\documentclass[twoside,11pt]{article}

\usepackage{jmlr2e}

\usepackage{microtype}
\usepackage{graphicx}
\usepackage{subfigure}
\usepackage{booktabs}

\usepackage{notations}
\usepackage{mathletters}

\usepackage{hyperref}

\usepackage[utf8]{inputenc} % allow utf-8 input
\usepackage[T1]{fontenc}    % use 8-bit T1 fonts
\hypersetup{
    colorlinks = true,
    allcolors = blue
}
\usepackage{url}            % simple URL typesetting
\usepackage{booktabs}       % professional-quality tables
\usepackage{amsfonts}       % blackboard math symbols
\usepackage{nicefrac}       % compact symbols for 1/2, etc.
\usepackage{microtype}      % microtypography
\usepackage{times}
\usepackage[scaled=0.8]{beramono}
\usepackage{fancybox}
\usepackage{changepage}
\usepackage{enumitem}
\usepackage{comment}
\usepackage{notations}
\usepackage{stmaryrd}
\usepackage{bold-extra}
\usepackage{xcolor}

\usepackage{appendix} 

\usepackage{algorithm,algorithmic}

\usepackage{graphicx}
\graphicspath{{.}{figs/}}
\usepackage{sidecap}
\usepackage{amsmath,amsfonts,amssymb, tikz}
\usepackage{mathletters} 
\usetikzlibrary{arrows,chains,matrix,positioning,scopes}

\usepackage{mathtools}
\usepackage{xspace}

\newcommand{\beq}{\begin{equation}}
\newcommand{\eeq}{\end{equation}}
\newcommand{\beqa}{\begin{eqnarray}}
\newcommand{\eeqa}{\end{eqnarray}}
\newcommand{\beqan}{\begin{eqnarray*}}
	\newcommand{\eeqan}{\end{eqnarray*}}

\newcommand{\beqannumb}{\begin{eqnarray}}
\newcommand{\eeqannumb}{\end{eqnarray}}
\newcommand{\hl}[1]{{\textit{#1}}}

\usepackage[colorinlistoftodos, textwidth=4cm, shadow]{todonotes}

\firstpageno{1}

\begin{document}

\title{Forced-exploration free Strategies for Unimodal Bandits}

\author{\name Hassan Saber \email hassan.saber@inria.fr \\
       \addr SequeL Research Group \\
       Inria Lille-Nord Europe \& CRIStAL\\
       Villeneuve-d'Ascq, Parc scientifique de la Haute-Borne, France
       \AND
       \name Pierre Ménard \email pierre.menard@inria.fr \\
       \addr SequeL Research Group \\
       Inria Lille-Nord Europe \& CRIStAL\\
       Villeneuve-d'Ascq, Parc scientifique de la Haute-Borne, France
       \AND
       \name Odalric-Ambrym Maillard \email odalric.maillard@inria.fr \\
       \addr SequeL Research Group \\
       Inria Lille-Nord Europe \& CRIStAL\\
       Villeneuve-d'Ascq, Parc scientifique de la Haute-Borne, France}

\editor{}

\maketitle

\begin{abstract}
We consider a multi-armed bandit problem specified by a set of Gaussian or  Bernoulli distributions endowed with a unimodal structure. 
	  Although this problem has been addressed in the literature \citep{combes2014unimodal}, the state-of-the-art algorithms for such structure make appear a forced-exploration mechanism.
	    We introduce \IMEDUB, the first forced-exploration free strategy that exploits the unimodal-structure, by adapting to this setting the Indexed Minimum Empirical Divergence (\IMED) strategy introduced by \citet{honda2015imed}.
	    This strategy is proven optimal.  We then derive \KLUCBUB, a \KLUCB version of \IMEDUB, which is also proven optimal. 
	    Owing to our proof technique, we are further able to provide a concise finite-time analysis of both strategies in an unified way.
	    Numerical experiments show that both \IMEDUB and \KLUCBUB perform similarly in practice and outperform the state-of-the-art algorithms.
 
\end{abstract}


	\section{Introduction}
	
	The multi-armed bandit problem is a popular framework to formalize sequential decision making problems.
	It was first introduced in the context of medical trials \citep{thompson1933likelihood,thompson1935criterion} and later formalized by \citet{ro52}:
	A bandit is specified by a set of unknown probability distributions $\nu  \!=\! (\nu_a)_{a\in\cA}$ with means $(\mu_{a})_{a\in\cA}$.
	At each time $ t \!\in\!\Nat $, the learner chooses an arm $ a_t \!\in\! \cA $, based only on the past,
	the learner then receives and observes a reward $ X_t $, conditionally independent, sampled according to $ \nu_{a_t} $. The goal of the learner is  to maximize the expected sum of rewards received over time (up to some unknown horizon $T$), or 
	equivalently minimize the \hl{regret} with respect to the strategy constantly receiving the highest mean reward 
	$$  R(\nu,T) = \Esp_\nu\!\brackets{\sum_{t=1}^T \mu^\star - X_t}  \text{ where }  \mu^\star=\max_{a \in \cA}\mu_{a}\,. $$ 
	Both means and distributions are \hl{unknown}, which makes the problem non trivial, and the learner only knows that $\nu\!\in\!\cD$  where $\cD$ is a given set of bandit configurations.
	This problem received increased attention in the middle of the $20^{\text{th}}$ century, and the seminal paper \citet{lai1985asymptotically} established the first  lower bound on the cumulative regret, showing that designing a strategy 
	that is optimal uniformly over a given set of configurations $\cD$ comes with a price. 
	% For instance 
	% for the set $\cB = \{ \nu_\theta=(\Bern(\theta_a))_{a\in\cA} : \theta_a \in[0,1] \}$ of bandit configurations with Bernoulli distributions, any uniformly-good strategy on $\cB$ must satisfy
	%  \beqan
	%  \forall \nu\in\cB,
	%  \liminf_{T\to\infty}  \frac{\cR(\nu,T)}{\log(T)} \geq \sum_{a\in\cA}\frac{\mu^\star(\nu)- \mu_a(\nu)}{\KL(\Bern(\mu_a(\nu),\Bern(\mu^\star(\nu)))}\,.
	%  \eeqan
	The study of the lower performance bounds in multi-armed bandits successfully lead to the development of asymptotically optimal strategies for specific configuration sets, such as the \KLUCB strategy \citep{lai1987adaptive,CaGaMaMuSt2013,maillard2018boundary} for exponential families, or alternatively the \DMED and \IMED strategies from \citet{honda2011asymptotically,honda2015imed}.
	The lower bounds from \citet{lai1985asymptotically}, later extended by \citet{burnetas1997optimal} did not cover all possible configurations, and in particular \hl{structured} configuration sets were not  handled until \citet{agrawal1989asymptotically} and then \citet{graves1997asymptotically} established generic lower bounds. Here, structure refers to the fact that pulling an arm may reveals information that enables to refine estimation of other arms. %(this is not the case in $\cB$). 
	Unfortunately, designing  numerical efficient strategies that are provably optimal remains a challenge for many structures.

	\paragraph{Structured configurations.}
	Motivated by the growing popularity of bandits in a number of industrial and societal application domains, the study of \hl{structured configuration sets} has received increasing attention over the last few years:
	The linear bandit problem is one typical illustration \citep{abbasi2011improved, srinivas2010gaussian, durand2017streaming}, for which the linear structure considerably modifies the achievable lower bound, see \citet{lattimore2017end}. 
	The study of a \hl{unimodal} structure naturally appears in many contexts, e.g.  single-peak preference economics, voting theory
	or wireless communications, and has been first considered in \citet{yu2011unimodal} from a bandit perspective, then in  \citet{combes2014unimodal}  providing an explicit lower bound together with a strategy exploiting this specific structure.
	Other structures  include Lipschitz bandits \cite{magureanu2014oslb}, and we refer to the manuscript \citet{magureanu2018efficient} for other examples, such as cascading bandits that are useful in the context of recommender systems.
	In \citet{combes2017minimal}, a generic strategy is introduced called \OSSB (Optimal Structured Stochastic Bandit), stepping the path towards generic multi-armed bandit strategies that are  adaptive to a given structure.

	\paragraph{Unimodal-structure.}
	In this paper, we provide novel regret minimization results related to the following structure. We assume a \hl{unimodal} structure similar to that considered in \citet{yu2011unimodal} and  \citet{combes2014unimodal}. That is, there exists an undirected graph $G \!=\! (\cA, E)$ whose vertices are arms $\cA$, and whose edges $E$ characterize a partial order  among means $(\mu_a)_{a\in \cA}$. This partial order is assumed unknown to the learner. We assume that there exists a unique optimal arm $ a^\star\!=\!\argmax_{a \in \cA}\mu_a$ and that for  all sub-optimal arm $a\!\neq\! a^\star$, there exists a path $P_a \!=\! (a_1 \!=\! a, \dots, a_{\ell_a} \!=\! a^\star) \!\in\! \cA^{\ell_a} $ of length $\ell_a \!\geq\! 2$ such that for all $ i\!\in\! [1,\ell_a -1]$, $(a_i, a_{i+1}) \in E$ and $\mu_{a_i} < \mu_{a_{i+1}}$. Lastly, we assume that $\nu \!\subset\!\cP\!\coloneqq\!\Set{p(\mu), \mu\!\in\!\Theta}$, where  $p(\mu)$ is an exponential-family distribution probability with density $f(\cdot, \mu)$ with respect to some positive measure $\lambda$ on $\Real$ and mean $\mu \!\in\!\Theta \!\subset\!\Real$. $\cP$ is assumed to be known to the learner. Thus, for all $a \!\in\! \cA$ we have $\nu_a \!=\! p(\mu_a)$. %In the following, 
	We denote by $\cD_{(\cP,G)}$ or simply $\cD$ the structured set of such unimodal-bandit distributions characterized by $\left(\cP,G\right)$. In the following, we assume that $\cP$ is either the set of real Gaussian distributions with means in $\Real$ and variance $1$ or the set of Bernouilli distributions with means in $(0,1)$. 
	
	% Graph and recommender systems 
	% \cite{lazaric2013sequential} Transfered
	% \cite{gopalan2016low} 
	% \cite{valko:tel-01359757}
	
	%$$ \Rightarrow schema $$

	\paragraph{Goal.}
	A key contribution in the study of unimodal bandits is the work \citet{combes2014unimodal}, where the authors establish lower confidence bounds on the regret for the unimodal structure, and introduce an asymptotically optimal strategy called \OSUB. One may then consider that unimodal bandits are solved.
    Unfortunately, a closer look at the proposed approach reveals that the considered strategy forces some arms to be played (this is different than what is called forced exploration in structured bandits; it is rather a forced exploitation scheme).
    In this paper, our goal is to introduce alternative strategies to \OSUB, that do not use any such forcing scheme, but consider variants of the pseudo-index induced by the lower bound analysis. Whether or not forcing mechanisms are desirable features is currently still under debate in the community; by providing the first strategy without any requirement for forcing in a structured bandit setup, we show that such mechanisms are not always required, which we believe opens an interesting avenue of research. 

	\paragraph{Contributions.} In this paper, we first revisit the Indexed Minimum Empirical Divergence (\IMED) strategy from \citet{honda2011asymptotically} introduced for unstructured multi-armed bandits, and adapt it to the unimodal-structured setting. We introduce  in Section~\ref{sec:imed_algo} the \IMEDUB strategy that is limited to the pulling of the current best arm or their no more than $d$ nearest arms at each time step, with $d$ the maximum degree of nodes in $G$.
	Being constructed from \IMED, \IMEDUB does not require any optimization procedure and does not separate exploration from exploitation rounds. \IMEDUB appears to be a \textit{local} strategy. Motivated by practical considerations, under the assumption that $G$ is a tree, when the number of arms $\abs{\cA}$ becomes large, we further develop \dIMEDUB, an algorithm 
	%(for dichotomous \IMEDUB)
	that behaves like \IMEDUB while resorting to a dichotomic second order exploration over all nodes of the graph. This helps quickly identify the best arm $a^\star$ within a large set of arms $\cA$ by empirical considerations.   We also introduce for completeness the \KLUCBUB strategy, that is similar to \IMEDUB, but inspired from \UCB strategies. %More precisely it uses Kullback–Leibler Upper Confidence Bounds (\KLUCB) of the means (the computed bound is a generalized index, as for \IMED, in the sense it depends not only on the arms, but also on the best empirical arm). 
	We prove in Theorem~\ref{th:asymptotic_optimality} that \IMEDUB, \dIMEDUB and \KLUCBUB are asymptotically optimal strategies that do not require forcing scheme. Furthermore, our unified finite time analysis shows that \IMEDUB and \KLUCBUB are closely related. Furthermore, these novel strategies significantly outperform \OSUB in practice. This is confirmed by numerical illustrations on synthetic data.  We believe that the construction of these algorithms together with the proof techniques developed in this paper are of independent interest for the bandit community.

	\paragraph{Notations.} Let $\nu \!\in\! \cD$. Let $\mu^\star \!=\! \max_{a \in \cA }\mu_a $ be the optimal mean  and $a^\star \!=\! \argmax_{a \in \cA}{\mu_a}$ be the optimal arm of $\nu$. We define for an arm $a\!\in\! \cA$ its sub-optimality gap $\Delta_a \!=\! \mu^\star \!-\! \mu_a$.  Considering an horizon $T\!\geq\! 1$, thanks to the chain rule we can rewrite the regret as follows:
	\begin{equation}
	R(\nu,T) = \sum_{a \in \cA} \Delta_a\, \Esp_\nu\big[N_a(T)\big]\,,
	\label{eq:chain_rule}
	\end{equation}
	where $ N_a(t) \!=\! \sum_{s=1}^t \ind_{\Set{a_s = a }} $ is the number of pulls of arm $a$  at time $t$.
	
	\section{Regret Lower bound}
	\label{sec:lower_bounds}
	
	In this subsection, we recall for completeness the known lower bound on the regret when we assume a unimodal structure. 	In order to obtain non trivial lower bound we consider
	%, as in the classical bandit problem, 
	strategies that are \hl{consistent} (aka uniformly-good).
	\begin{definition}[Consistent strategy]\label{def:consistent}
		A strategy is consistent on $\cD$ if for all configuration $\nu\in \cD$, for all sub-optimal arm $a$, for all $ \alpha >  0$, 
		\[
		\limT\Esp_\nu \!\left[\dfrac{N_a(T)}{T^\alpha}\right] = 0\,. 
		\]
	\end{definition} 
	We can derive from the notion of consistency an asymptotic lower bound on the regret, see \citet{combes2014unimodal}. To this end, we introduce $\cV_a \!=\! \Set{a' \in \cA:\ (a,a') \in E} $ to denote the neighbourhood of an arm $a \in \cA$. 
	\begin{proposition}[Lower bounds on the regret]\label{prop:LB_regret}Let us consider a  consistent strategy. Then, for all configuration $ \nu \!\in\! \cD$, it must be that
		\[
		\liminfT \dfrac{R(\nu,T)}{\log(T)} \geq c(\nu):= \sum_{a \in \cV_{a^\star}} \dfrac{\Delta_a}{\KLof{\mu_a}{\mu^\star}} \,,
		\]
		where  $\KLof{\mu}{\mu'} \!=\!\int_{\Real}\!\log\!\left(f(x,\mu)/f(x,\mu')\right)\!f(x,\mu) \lambda(\mathrm{d}x)$ denotes the Kullback-Leibler divergence between $\nu\!=\!p(\mu)$ and $\nu'\!=\!p(\mu')$, for $\mu,\mu' \!\in\! \Theta$. 
		\label{prop:lower_bound}
	\end{proposition}
	\begin{remark} The quantity $c(\nu)$ is a fully explicit function of $\nu$ (it does not require solving any optimization problem) for some set of distributions $\nu$ (see Remark~\ref{lb Bern}).
	This useful property no longer holds in general for arbitrary structures. Also, it is noticeable that $c(\nu)$ does not involve all the sub-optimal arms but only the ones in $\cV_{a^\star}$. This indicates that sub-optimal arms outside $\cV_{a^\star}$ are sampled $o(\log(T))$, which contrasts with the unstructured stochastic multi-armed bandits. 	See \citet{combes2014unimodal} for further insights.
	\end{remark}
	\begin{remark} \label{lb Bern} For Gaussian distributions (variance $\sigma^2 \!=\!1$), we assume $\lambda$ to be the Lebesgue measure, $\Theta\!=\!\Real$, and for $\mu \!\in\!\Real$, $f(\cdot,\mu)\!= : x \!\in\!\Real \mapsto (\sqrt{2\pi})^{-1}e^{-(x-\mu)^2\!/2}$. Then for all $\mu,\mu'\!\in\!\Real$,  $ \KLof{\mu}{\mu'}\!=\! (\mu' \!-\! \mu)^2\!/2 $. For Bernoulli distributions, a  possible setting is to assume $\lambda = \delta_0 + \delta_1$ (with $\delta_0, \delta_1$ Dirac measures), $\Theta\!=\!(0,1)$ and for $\mu\!\in\!\Theta$, $f(\cdot,\mu)\!=: x \!\in\!\Set{0,1} \mapsto \mu^x(1-\mu)^{1-x}$. Then for all $\mu,\mu'\!\in\![0,1]$, $\KLof{\mu}{\mu'}\!=\! \klof{\mu}{\mu'}$, where
	$$ \klof{\mu}{\mu'}\!\coloneqq\! \left\{\begin{array}{ll}
	     \!\!\!0 &\hspace{-5mm}\textnormal{if } \mu\!=\! \mu',    \\
	     \!\!\!+ \infty &\hspace{-5mm}\textnormal{if } \mu \!<\! \mu' \!=\! 1,    \\
	     \!\!\!\mu\log\!\left(\frac{\mu}{\mu'}\right)+(1\!-\!\mu)\log\!\left(\frac{1\!-\!\mu}{1\!-\!\mu'}\right)&\hspace{-3mm}\textnormal{otherwise},    \\     
	\end{array} \right.  $$
	with the convention $0\!\times\!\log(0) \!=\! 0$.
	\end{remark}
	
	\section{Forced-exploration free strategies for unimodal-structured bandits}
	\label{sec:imed_algo}
	
	We present in this section three novel strategies that both match the asymptotic lower bound of Proposition~\ref{prop:LB_regret}. Two of these strategies are inspired by the Indexed Minimum Empirical Divergence (\IMED) proposed by \citet{honda2011asymptotically}. The other one is based on Kullback–Leibler Upper Confidence Bounds (\KLUCB), using insights from \IMED.
	The general idea behind these algorithms is, following the intuition given by the lower bound, to narrow on the current best arm and its neighbourhood for pulling an arm at a given time step.
	\paragraph{Notations.} The empirical mean of the rewards from the arm $a$ is denoted by $ \muhat_a(t) \!=\!\sum_{s=1}^ t{\ind_{\Set{ a_s = a }} X_s}/N_a(t) $ if $ N_a(t)\!>\! 0  $, $ 0 $ otherwise. We also denote by $\muhat^\star(t) \!=\! \max_{a\in\cA}\muhat_a(t)$ and $\Ahat^\star(t) \!=\! \argmax_{a \in \cA}\muhat_a(t)$ respectively the current best mean and the current set of optimal arms. 
	
	For convenience, we recall below the \OSUB  (Optimal sampling for Unimodal Bandits) strategy from \citet{combes2014unimodal}.
		\begin{algorithm}[H]
		\caption{\OSUB }
		\label{alg:osub}
		\begin{algorithmic}
		%	\REQUIRE  $\left(\nu \textnormal{ is unimodal}\right)$ is true.
		    \STATE Pull an arbitrary arm $a_1\in \cA$
			\FOR{$ t = 1 \dots T-1$}
			\STATE Choose $\ahat^\star_t \in \argmin\limits_{\ahat^\star \in \Ahat^\star(t) }N_{\ahat^\star}(t) $ (chosen arbitrarily)
			\STATE Pull $a_{t+1}  = \begin{cases}
			    \ahat^\star_t  & \text{ if }\frac{L_{t}(\ahat^\star_t)-1}{d + 1} \in \Nat\\
			    \argmax\limits_{a\in \cV_{\ahat^\star_t}} u_a(t)  & \text{else}
			    \end{cases}$
			\ENDFOR
		\end{algorithmic}
	\end{algorithm}
	\noindent In Algorithm~\ref{alg:osub}, for some numerical constant $c\!>\!0$, the index computed by \OSUB strategy for arm $a\!\in\!\cA$ and step $t\!\geq\!1$ is 
	\[
	u_a(t)\!=\! \sup\big\{ u \!\geq\! \muhat_a(t)\!: N_a(t)\KLof{\muhat_a(t)}{u} \!\leq\! f_c\!\left(L_t(\ahat^\star_t)\right) ,
	\]
	where  $L_t(a) \!=\! \sum_{t'=1}^t\ind_{\Set{\ahat^\star_{t'}=a }}$ counts how many times  arm $a$ was a leader (best empirical arm), $d$ is the maximum degree of nodes in $G$, and $f_c(\cdot)\!=\!\log(\cdot)\!+\!c\log\log(\cdot)$.
	%\oam{One should double check the pseudo-code, especially the $t$ vs $t-1$ things}
	%\hs{I think it's $L_t(\ahat^\star_t)$ and not $L_t(a)$ for the computation of $u_a(t)$. No ?}
	
	\subsection{The \IMEDUB strategy.}
	For all arm $ a \!\in\! \cA$ and time step $t \!\geq\! 1$ we introduce the \IMED index 
	$$ I_a(t) = N_a(t) \, \KLof{\muhat_a(t)}{\muhat^\star(t)} + \log\!\left(N_a(t)\right) \,, $$ with the convention $0\!\times\!\infty \!=\! 0$.
	This index can be seen as a transportation cost for moving a sub-optimal arm to an optimal one plus an exploration term: the logarithm of the numbers of pulls. When an optimal arm is considered, the transportation cost is null and there is only the exploration part. Note that, as stated in \citet{honda2011asymptotically}, $I_{a}(t)$ is an index in the weaker sense since it cannot be determined only by samples from the arm $a$ but also uses empirical means of current optimal arms. We define \IMEDUB (Indexed Minimum Empirical Divergence for Unimodal Bandits), described in Algorithm~\ref{alg:imedub}, to be the strategy consisting of pulling an arm  $a_t \!\in\! \Set{\ahat^\star_t}\!\cup\!\cV_{\ahat^\star_t}$ with minimum index at each time step $t$, where is $\ahat^\star_t \!\in\! \argmin_{\ahat^\star \in \Ahat^\star(t) }N_{\ahat^\star}(t)$ is a current best arm. This is a natural algorithm since the lower bound on the regret given in Proposition~\ref{prop:LB_regret} involves only the arms in $\cV_{a^\star}$, the neighbourhood of the arm $a^\star$ of maximal mean.
	\begin{algorithm}[H]
		\caption{\IMEDUB}
		\label{alg:imedub}
		\begin{algorithmic}
		    \STATE Pull an arbitrary arm $a_1\in \cA$
			\FOR{$ t = 1 \dots T-1$}
			\STATE Choose $\ahat^\star_t \in \argmin\limits_{\ahat^\star \in \Ahat^\star(t) }N_{\ahat^\star}(t) $ (chosen arbitrarily)
			\STATE  Pull $a_{t+1} \in \argmin\limits_{a \in \Set{\ahat^\star_t}\cup\cV_{\ahat^\star_t}}I_a(t)$ (chosen arbitrarily)
			\ENDFOR
		\end{algorithmic}
	\end{algorithm}

	\subsection{The \KLUCBUB strategy}
	For all arm $ a \!\in\! \cA$ and time step $t \!\geq\! 1$ we introduce the following Upper Confidence Bound
	{\small\[
	U_a(t)\!=\! \max\!\left\{ \begin{array}{l}
	    \hspace{-2mm} u \geq \muhat_a(t)\\ 
	    \hspace{-2mm}  N_a(t)  \KLof{\muhat_a(t)}{u} \!+\! \log\!\left(N_a(t)\right) \!\leq\! \log\!\left(N_{\ahat_t^\star}(t)\right) 
	\end{array} \hspace{-3mm} \right\}  
	\]}
	with $\ahat^\star_t \!\in\! \argmin\limits_{\ahat^\star \in \Ahat^\star(t)}\!N_{\ahat^\star}(t)$. 
	
	By convention, we set $U_a(t) \!=\! \muhat_a(t)$ if for $a\! \in\! \cA$, $\log\!\left(N_a(t)\right) \!>\! \log\!\left(N_{\ahat_t^\star}(t)\right)$.
	\begin{remark}
	A classical \KLUCB strategy would replace the term $\log\!\left(N_{\ahat_t^\star}(t)/N_a(t)\right)$ with $\log(t)$, 
	and a \KLUCBp would use $\log(t/N_a(t))$. This is a simple yet crucial modification. Indeed, although this makes \KLUCBUB not an index strategy, this enables to get a more intrinsic strategy, to simplify the analysis and get improved numerical results.
	\end{remark}
	 As for \IMEDUB and \IMED,  $U_a(t)$ is an index in a weaker sense since it cannot be determined only by samples from the arm $a$ but also uses numbers of pulls of current optimal arms. We define \wucbub (Kullback-Leibler Upper Confidence Bounds for Unimodal Bandits) to be the strategy consisting of pulling an arm  $a_t \!\in\! \Set{\ahat^\star_t}\cup\cV_{\ahat^\star_t}$ with maximum index at each time step $t$. This algorithm can be seen as a \KLUCB version of the \IMEDUB strategy.
	\begin{algorithm}[H]
		\caption{\KLUCBUB}
		
		\label{alg:wucbub}
		\begin{algorithmic}
            \STATE Pull $a_1\in \cA$ at random.
			\FOR{$ t = 1 \dots T-1$}
			\STATE Choose $\ahat^\star_t \in \argmin\limits_{\ahat^\star \in \Ahat^\star(t) }N_{\ahat^\star}(t) $ (chosen arbitrarily)
			\STATE  Pull $a_{t+1} \in \argmax\limits_{a \in \Set{\ahat^\star_t}\cup\cV_{\ahat^\star_t}}U_a(t)$ (chosen arbitrarily)
			\ENDFOR
		\end{algorithmic}
	\end{algorithm}
	\begin{remark} \IMEDUB does not require solving any optimization problem, unlike \OSUB or \KLUCBUB. We believe this feature, inherited from \IMED, makes it an especially appealing strategy. \KLUCBUB solves an optimization similar to that of the \KLUCB strategy for unstructured bandits, and also related to the optimization used in \OSUB from \citet{combes2014unimodal}. The difference between \KLUCBUB and \OSUB is that it does not use any forced exploitation.
	\end{remark}

	\subsection{The \dIMEDUB strategy for large set of arms}
	When the set of arms is large, a bad initialization of \IMEDUB (that is, choose arm $a_1$ far from $a^\star$) comes with high initial regret.  
	Indeed, \IMEDUB does not allow to explore outside the neighbourhood $\cV_{\ahat^\star_t}$ of $\ahat^\star_t$. When $\cA$ is large compared to the neighbourhoods, this may generate a large burn-in phase.
	%As a result, under \IMEDUB and for a large set of arms $\cA$, a time step $t \!\geq\!1 $ such  that $\ahat^\star_t$  poorly estimates $a^\star$ has a very negative impact on the regret. 
	%Then, for a large set of arms, a bad initialization of \IMEDUB (that is, choose arm $a_1$ far from $a^\star$) comes with high regret. 
	To overcome this practical limitation, it is natural to explore outside the neighbourhood of the current best arm. However, to be compatible with the lower bound on the regret stated in Proposition~\ref{prop:lower_bound} such exploration must be asymptotically negligible. 
	We now consider $\cG$ to be a tree, and introduce \dIMEDUB, a strategy that trades-off between these two types of exploration. \dIMEDUB shares with \IMEDUB the same exploitation criteria and explores if the index of the current best arm exceeds the indexes of arms in its neighbourhood. However, in exploration phase, \dIMEDUB runs an \IMED type strategy to choose between exploring within \emph{or outside} the neighbourhood of the current best arm. For all time step $t \!\geq\! 1$, for all arm $ a' \!\in\! \cV_{\ahat^\star_t}$, for all arm $a \!\in\!\hat G_{a'}(t)$, where $G_{a'}(t)$ denotes the sub-tree containing $a'$ obtained by cutting edge $(a',\ahat^\star_t)$, we define the second order \IMED index relative to $a'$, as
	$$ I_a^{(a')}(t) = N_a(t) \, \KLpof{\muhat_a(t)}{\muhat_{a'}(t)} + \log\!\left(N_a(t)\right) \,, $$ 
	where $\KLpof{\mu}{\mu'}\!=\!\KLof{\mu}{\mu'}$ if $\mu \!<\! \mu'$, $0$ otherwise. At each exploration time step, \dIMEDUB pulls an arm in $\cS_t$ with minimal secondary index relative to the arm $\aul_t$ with current minimal index and belonging to the neighbourhood of the current best arm, where $\cS_t$ is a sub-tree of $ \hat G_{\aul_t}(t)$ dichotomously chosen that contains $\aul_t$. 
	We illustrate in Appendix~\ref{app: exp}, a way to dynamically choose $\cS_t$.  
	\begin{remark}Assuming that $G$ is a tree ensures that for all $a'\!\in\!\cV_{a^\star}$, the nodes of $G_{a'}$, the sub-tree containing $a'$ obtained by cutting edge $(a',a^\star)$, induce a unimodal bandit configuration with optimal arm $a'$. This specific property allows establishing the optimality of $\dIMEDUB$.
	\end{remark}
	\begin{algorithm}[H]
		\caption{\dIMEDUB}
		\label{alg:dimedub}
		\begin{algorithmic}
		    \STATE Pull an arbitrary arm $a_1\in \cA$
			\FOR{$ t = 1 \dots T-1$}
			\STATE Choose $\ahat^\star_t \in \argmin\limits_{\ahat^\star \in \Ahat^\star(t) }N_{\ahat^\star}(t) $ (chosen arbitrarily)
			\STATE  Choose $\aul_t \in \argmin\limits_{a \in \Set{\ahat^\star_t}\cup\cV_{\ahat^\star_t}}I_a(t)$ (chosen arbitrarily)
			\IF{$\aul_t = \ahat^\star_t$}
			\STATE Pull $a_{t+1} = \aul_t$
			\ELSE
			\STATE Pull $ a_{t+1} \in \argmin\limits_{a \in \cS_t}I^{(\aul_t)}_{a}(t) $
			\ENDIF
			\ENDFOR
		\end{algorithmic}
	\end{algorithm}
	\subsection{Asymptotic optimality of \IMEDUB, \dIMEDUB and \KLUCBUB}
	
	In this section, we state the main theoretical result of this paper.
	\begin{theorem}[Upper bounds] \label{th:upper bounds}  Let us consider a set of Gaussian or Bernoulli distributions $ \nu \!\in\! \cD$ and let $a^\star$ its optimal arm. Let $\cV_{a^\star}$  be the  sub-optimal arms in the neighbourhood of $a^\star$. Then under \IMEDUB and \KLUCBUB strategies for all $0 \!<\! \epsilon \!<\! \epsilon_\nu $, for all horizon time $ T \!\geq\! 1$, for all $a \!\in\!\cV_{a^\star}$,
		\[
		\Esp_\nu[N_{a}(T)] \leq \dfrac{1 + \alpha_\nu(\epsilon)}{\KLof{\mu_a}{\mu_{a^\star}}} \log(T) + d\abs{\cA}^2C_\epsilon + 1 
		\]
		and, for all $a \!\notin\! \Set{a^\star}\!\cup\!\cV_{a^\star} $, 
		\[ \Esp_\nu[N_{a}(T)] \leq d\abs{\cA}^2C_\epsilon + 1 \,,
		\]
		where  $d$ is the maximum degree of nodes in $G$, $\epsilon_\nu \!=\! \min \Set{ 1\!-\!\mu^\star,\, \min_{a \neq a' }\abs{\mu_a \!-\! \mu_{a'}}\!/\!4}$, $C_\epsilon \!=\! 34\log(1\!/\!\epsilon)\epsilon^{-6}$ and where $\alpha_\nu(\cdot)$ is a non-negative function depending only on $\nu$ such that $\lim\limits_{\epsilon \to 0}\alpha_\nu(\epsilon)\!=\!0$ (see Section~\ref{imed_unimodal notations} for more details).
		
		Furthermore, if the considered graph is a tree, then  under \dIMEDUB, for all horizon $T \!\geq\!1$, for all $a \in \cV_{a^\star}$,
\[
		\Esp_\nu[N_{a}(T)] \leq \dfrac{1 + \alpha_\nu(\epsilon)}{\KLof{\mu_a}{\mu_{a^\star}}} \log(T) + d\abs{\cA}^2C_\epsilon + 1 
		\]
	and, for all $a \!\notin\! \Set{a^\star}\!\cup\!\cV_{a^\star} $, 
	\beqan
	\Esp_\nu\!\left[N_{a}(T)\right] \hspace{-3mm}&\leq&\hspace{-3mm}
	\dfrac{1 + \alpha_\nu(\epsilon)}{\min\limits_{\aul \in \cV_{a^\star}}\!\!\KLof{\mu_a}{\mu_{\aul}}} \log\!\!\left(\!\dfrac{1 + \alpha_\nu(\epsilon)}{\min\limits_{\aul \in \cV_{a^\star}}\!\!\KLof{\mu_{\aul}}{\mu_{a^\star}\!}}\! \log(T)\!\!\right) \\
	&+& \hspace{-3mm} d\abs{\cA}^2C_\epsilon + 1 \,.
	\eeqan
		
	\end{theorem}
	In particular one can note that the arms in the neighbourhood of the optimal one are pulled $\cO\!\left(\log(T)\right)$ times while the other sub-optimal arms are pulled a finite number of times under \IMEDUB and \KLUCBUB, and $\cO\!\left(\log\!\log(T)\right)$ times under \dIMEDUB. This is coherent with the lower bound that only involves the neighbourhood of the best arm.
	More precisely, combining Theorem~\ref{th:upper bounds} and the chain rule~\eqref{eq:chain_rule} gives the asymptotic optimality of \IMEDUB and \wucbub with respect to the lower bound of Proposition~\ref{prop:LB_regret}.
	\begin{corollary}[Asymptotic optimality]With the same notations as in Theorem~\ref{th:upper bounds} , then under \IMEDUB and \KLUCBUB strategies
		\[
		\limsupT \dfrac{R(\nu,T)}{\log(T)} \leq c(\nu) = \sum\limits_{a \in \cV_{a^\star} } \dfrac{\Delta_a}{\KLof{\mu_a}{\mu^\star}} \,.
		\]
		If the considered graph is a tree, same result holds under \dIMEDUB strategy.
% 		Furthermore,
% 		\[ \forall a \in \cV_{a^\star},\quad \limsupT\dfrac{\Esp_\nu[N_a(T)]}{\log(T)} \leq \dfrac{1}{\KL(\mu_a|\mu_{a^\star})} 
% 		\]
% 		and 
% 		\[\forall a \notin \Set{a^\star}\cup\cV_{a^\star},\quad \limsupT\Esp_\nu[N_a(T)] \leq \dfrac{97\abs{\cA}}{\epsilon_\nu^4} + 1 \,,
% 		\]
% 		where $\epsilon_\nu = \min\limits_{a \in \cA, a' \in \cV_a}\dfrac{\abs{\mu_a - \mu_{a'}}}{4}$.
		\label{th:asymptotic_optimality}
	\end{corollary} 
	See respectively Section~\ref{sec : imed_analysis} and Appendix~\ref{app: ucb_analysis} for a finite time analysis of \IMEDUB, \dIMEDUB and \KLUCBUB. 
% 	With regards to the lower bound of Proposition~\ref{prop:lower_bound} this theorem proves via the chain rule~\eqref{eq:chain_rule} that both strategies are \emph{asymptotically optimal}. In particular it worth to nate that the two arms in the neighbourhood of the optimal one are pulled $\cO\left(\log(T)\right)$ times while the other sub-optimal arms are pulled a finite number of times. This is coherent with the lower bound that only involves the neighbourhood of the best arm.   
	
	% Under both \imedub and WUCB strategies the neighbours of the best arm are thus pulled $\cO\left(\log(T)\right)$ times while the other sub-optimal arms are pulled a finite number of times. This is coherent with the lower bound on the regret that only involves the neighbours of the best arm.   
	
	% \oam{Quickly highlight that this means $log(T)$ pulls for neighbours of optimal arm, and constant pull for other sub-optimal arms. Explain this is coherent with lower bound.}

	\section{\IMEDUB finite time analysis}
	\label{sec : imed_analysis}
	At a high level, the key interesting step of the proof is to realize that the considered strategies imply empirical lower and empirical upper bounds on the numbers of pulls (see Lemma~\ref{unimodal empirical lower bounds}, Lemma~\ref{unimodal empirical upper bounds} for \IMEDUB). Then, based on concentration lemmas (see Section~\ref{subsec : imed_concentration}), the strategy-based empirical lower bounds ensure the reliability of the estimators of interest (Lemma~\ref{unimodal reliability}). This makes use of more classical arguments based on concentration of measure. Then, combining the reliability of these estimators with the obtained strategy-base empirical upper bounds, we obtain upper bounds on the average numbers of pulls (Theorem~\ref{th:upper bounds}). 
	
	In this section, we only detail the finite time analysis of \IMEDUB algorithm and defer those of \dIMEDUB and \KLUCBUB to the appendix, as it follows essentially the same steps. Indeed,  we show that \KLUCB and \dIMEDUB strategies imply empirical bounds (Lemmas~\ref{ucb_unimodal empirical lower bounds},\ref{ucb_unimodal empirical upper bounds}, Lemmas~\ref{d-imedub unimodal empirical lower bounds},\ref{d-imed unimodal empirical upper bounds}) very similar to \IMEDUB strategy . This inequalities are the cornerstone of the analysis.  We believe that this general way of proceeding is of independent interest as it simplifies the proof steps.
	
	%\oam{Good, but perhaps highlight more what is novel here; we take a different path then ... etc Also explain that we only provide the main proof of IMED  here, and defer the one of UCB to the appendix, as it follows essentially the same ideas/steps ?}

	%\oam{We may be short on space to include the full proof in the paper. Would it be possible to 
	%	present all the lemmas but without their proof, and then explain how to combine them to obtain the final result? Then perhaps we can simply detail the most important lemmas and leave the more technical ones to the appendix?}
	
	\subsection{\label{imed_unimodal notations} Notations}
	Let us consider $\nu \!\in\! \cD$ and let us denote by  $a^\star$ its best arm. We recall that for all $a \!\in\! \cA $,  $\cV_{a} \!=\! \Set{a' \in \cA:\ (a,a') \in E}$ is the neighbourhood of arm $a$ in graph $G\!=\!(\cA,E)$, and that 
	\[
	d = \max\limits_{a \in\cA}\abs{\cV_a},\ \epsilon_\nu = \min \Set{ 1 - \mu^\star,\ \min\limits_{a  \neq a'}\dfrac{\abs{\mu_a - \mu_{a'}}}{4}} \,.
	\]
	Then, there exists a function $\alpha_\nu(\cdot)$ such that for all $a \neq a' $, for all $0 \!<\!\epsilon\!<\!\epsilon_\nu$, 
	\[
	 \dfrac{\klof{\mu_a }{ \mu_{a'}}}{1 + \alpha_\nu(\epsilon)}\leq \klof{\mu_a + \epsilon}{ \mu_{a'} - \epsilon} \leq (1 + \alpha_\nu(\epsilon)) \klof{\mu_a}{ \mu_{a'}}
	\]
	and $\lim\limits_{\epsilon\downarrow0}\downarrow\alpha_\nu(\epsilon) = 0$.
	For all studied strategy, at each time step $t \!\geq\! 1$, $\ahat^\star_t$ is arbitrarily chosen in $\argmin\limits_{a \in \Ahat^\star(t)}N_a(t)$ where $\Ahat^\star(t)\!=\!\argmax\limits_{a \in \cA}\muhat_a(t)$.  
	
	For all arms $a \!\in\! \cA$ and $n\!\geq\! 1$, we introduce the stopping times $\tau_{a,n} \!=\! \inf{ \Set{t \!\geq\! 1\!: N_{a}(t) \!=\! n} } $ and define the empirical means corresponding to local times
	\[
	\muhat_a^n = \dfrac{1}{n}\sum\limits_{m = 1}^n X_{\tau_{a,m}} \,.
	\]
	%Then for all arm $a \in \cA$ and local time $n \geq 1$, $\muhat_a^n$ follows the Gaussian law $\cN(\mu_a, \dfrac{1}{n})$. 
	For a subset of times $\cE \!\subset\! \Set{t \!\geq\! 1} $, we denote by $\cE^c$ its complementary in $\Set{t \geq 1}$.
	
	\subsection{Strategy-based empirical bounds} 
	\IMEDUB strategy implies inequalities between the indexes that can be rewritten as inequalities on the numbers of pulls. While  lower bounds involving $\log(t)$ may be expected in view of the asymptotic regret bounds, we show lower bounds on the numbers of pulls involving instead $\log\!\left(N_{a_{t+1}}(t)\right)$, the logarithm of the number of pulls of the current chosen arm. We also provide upper bounds on $N_{a_{t+1}}(t)$ involving $\log(t)$.
	%\oam{I don't understand the above sentence. Why "Unless ..., we show" ?  Should it be although?}

	We believe that establishing these empirical lower and upper bounds is a key element of our proof technique, that is of independent interest and not \textit{a priori} restricted to the unimodal structure.
	\begin{lemma}[Empirical lower bounds]\label{unimodal empirical lower bounds}Under \IMEDUB, at each step time $t \!\geq\! 1$, for all $a \!\in\!\cV_{\ahat_t^\star}$,
	\[\log\!\left(N_{a_{t+1}}(t)\right)  \leq N_{a}(t)\, \KLof{\muhat_{a}(t)}{\muhat^\star(t)} + \log\!\left(N_{a}(t)\right)
	\]
	and
	\[  N_{a_{t+1}}(t) \leq N_{\ahat^\star_t}(t)\,.
	\]
	\end{lemma}
	
	\begin{proof} For $a \!\in\! \cA$, by definition, we have $I_a(t) \!=\! N_{a}(t) \KLof{\muhat_{a}(t)}{\muhat^\star(t)} \!+\! \log\!\left(N_{a}(t)\right) $, hence
		\[
		\log\!\left(N_a(t)\right) \leq I_a(t) \,.
		\]
		This implies, since the arm with minimum index is pulled, $  \log\!\left(N_{a_{t+1}}(t)\right) \!\leq\! I_{a_{t+1}}(t) \!=\! \min\limits_{a' \in \Set{\ahat^\star_t}\!\cup\!\cV_{\ahat^\star_t}} I_{a'}(t) \!\leq\! I_{\ahat^\star_t}(t) \!=\! \log\!\left(N_{\ahat^\star_t}(t)\right)$. By taking the $\exp(\cdot)$, the last inequality allows us to conclude.
	\end{proof}
	\begin{lemma}[Empirical upper bounds]\label{unimodal empirical upper bounds}
		Under \IMEDUB at each step time $t \!\geq\! 1$,
		\[
		N_{a_{t+1}}(t) \,\KLof{\muhat_{a_{t+1}}(t)}{\muhat^\star(t)} \leq \log(t) \,. 
		\]
	\end{lemma}
	
	\begin{proof} As above, by construction we have
		\[
		I_{a_{t+1}}(t) \leq I_{\ahat^\star_t}(t) \,.  
		\]
		It remains, to conclude, to note that
		\[
		N_{a_{t+1}}(t) \KLof{\muhat_{a_{t+1}}(t)}{\muhat^\star(t)}  
		\leq I_{a_{t+1}}(t)\,, 
		\]
		and
		\[I_{\ahat^\star_t}(t) =  \log(N_{\ahat^\star_t}(t)) \leq \log(t) \,.
		\]
	\end{proof}
	
	\subsection{Reliable current best arm and  means}
	
	In this subsection, we consider the subset $\cT_\epsilon$ of times where everything is well behaved: The current best arm corresponds to the true one and the empirical means of the best arm and the current chosen arm are $\epsilon$-accurate for $0\!<\!\epsilon\!<\!\epsilon_\nu$, that is
	\[
	\cT_\epsilon \coloneqq \left\{\begin{array}{l}
		    \hspace{-2mm} t \geq 1:\  \Ahat^\star(t) = \Set{a^\star}  \\
		     \hspace{10mm}\forall a \in \Set{ a^\star,a_{t+1} },\ \abs{\muhat_a(t) - \mu_a } < \epsilon 
		\end{array} \! \right\}\,. 
		\]
	 We will show that its complementary set is finite on average. In order to prove this we decompose the set $\cT_\epsilon$ in the following way. Let $\cE_\epsilon$ be the set of times where the means are well estimated,
	 \[
		\cE_\epsilon \coloneqq \Set{t \geq 1:\ \forall a\! \in\! \Ahat^\star\!(t)\!\cup\!\Set{a_{t+1} },\ \abs{\muhat_a(t) - \mu_a } < \epsilon }\,, 
	  \]
	 and $\Lambda_\epsilon$ the set of times where an arm that is not the current optimal neither pulled is underestimated
	 {\small 
			\[
			\Lambda_\epsilon \!\coloneqq\! \!\left\{\begin{array}{l}
			\hspace{-2mm}t \!\geq\! 1 \!: \exists a \in \cV_{\ahat^\star_t}\!\setminus\!\Set{a_{t+1},\ahat^\star_t} \textnormal{ s.t. }  \muhat_a(t) \!<\! \mu_a \!-\! \epsilon \textnormal{ and}      \\
			 \hspace{-2mm}    \log(N_{a_{t+1}}(t)) \!\leq\!  N_a(t) \KL(\muhat_a(t)| \mu_a \!-\! \epsilon)  \!+\! \log\!\left(N_{a}(t)\right) 
			\end{array} \hspace{-2mm} \right\}\!.
			\]
		}
~\\Then we prove below the following inclusion.		
\begin{lemma}[\label{lem:decomposition_T_epsilon} Relations between the subsets of times]
For $ 0 \!<\! \epsilon \!<\! \epsilon_\nu$,
\begin{equation}
    \cT_\epsilon^c\setminus\cE_\epsilon^c  \subset \Lambda_\epsilon  \,.
\label{eq:decomp_T_epsilon}
\end{equation}
\end{lemma}
\begin{proof}
Let us consider $ t  \!\in\! \cT_\epsilon^c\!\setminus\!\cE_\epsilon^c$. Since $ t\!\in\! \cE_\epsilon$ and $\epsilon \!<\! \epsilon_\nu$ we have
\[
\forall a \in \Ahat^\star(t)\cup\Set{a_{t+1}},\quad \abs{\muhat_{a}(t) - \mu_a} < \epsilon \,.
\]
By triangle inequality this implies, for all $ \ahat^\star \!\in\! \Ahat^\star(t)$,
\beqan
\abs{\mu_{\ahat^\star_t}\! \!-\! \mu_{\ahat^\star}\!}\!\!-\! 2\epsilon \hspace{-3mm}&\leq& \hspace{-3mm}\abs{\mu_{\ahat^\star_t}\! \!-\! \mu_{\ahat^\star}\!}\! \!-\! \abs{\mu_{\ahat^\star_t}\! \!-\! \muhat_{\ahat^\star_t}\!(t)\!}\! \!-\! \abs{\muhat_{\ahat^\star}\!(t) \!-\! \mu_{\ahat^\star}\!} \\
&\leq&  \hspace{-3mm}\abs{\muhat_{\ahat^\star_t}\!(t)\! - \!\muhat_{\ahat^\star_t}(t)\!} \!\!=\! 0
\eeqan
and
\[
 \Ahat^\star(t) = \Set{\ahat^\star_t}  \,.
\]
Thus, since $t \!\notin\!\cT_\epsilon$, we have $ \ahat^\star_t \!\neq\! a^\star$. In particular, since $(\mu_a)_{a \in \cA}$ is unimodal, there exists $ a \in \cV_{\ahat^\star_t} $ such that $\mu_a \!>\! \mu_{\ahat^\star_t} $.  From Lemma~\ref{unimodal empirical lower bounds} we have the following empirical lower bound
\[
\log\!\left(N_{a_{t+1}}(t)\right)  \leq N_{a}(t) \KLof{\muhat_{a}(t)}{\muhat^\star(t)} + \log\!\left(N_{a}(t)\right)\,.
\]
Furthermore, since $ t \!\in\! \cE_\epsilon$ and $\epsilon \!<\! \epsilon_\nu$, we have
\[
\muhat_{a}(t)  \leq \muhat^\star(t) = \muhat_{\ahat^\star_t}(t) < \mu_{\ahat^\star_t} + \epsilon < \mu_{a} - \epsilon \,.
\]
Since $\abs{\muhat_{a_{t+1}}(t) - \mu_{a_{t+1}}} \!<\! \epsilon$, it indicates in particular that $a \!\in\! \cV_{\ahat^\star_t}\!\setminus\!\Set{a_{t+1}, \ahat^\star_t} $. In addition, the monotony of the $\KL(\cdot|\cdot)$ implies 
\[
\KLof{\muhat_{a}(t)}{\muhat^\star(t)} \leq \KLof{\muhat_{a}(t)}{  \mu_{a} - \epsilon} \,.
\]
Therefore for such $t$  we have $\muhat_{a}(t)  \!<\! \mu_a \!-\! \epsilon$ and
\[\log\!\left(N_{a_{t+1}}(t)\right)  \leq N_{a}(t) \KLof{\muhat_{a}(t)}{ \mu_a - \epsilon} + \log\!\left(N_{a}(t)\right) \,,
\]
which concludes the proof.
\end{proof}
We can now resort to classical concentration arguments in order to control the size of these sets, which 
yields the following upper bounds. We defer the proof to Appendix~\ref{app:proof_ce_and_lambda_finite} as they follow standard arguments.
\begin{lemma}[Bounded subsets of times]For $ 0 \!<\! \epsilon \!<\! \epsilon_\nu$, 
 \[\Esp_\nu[\abs{\cE_\epsilon^c}] \leq \dfrac{10\abs{\cA}^2}{\epsilon^4} \hspace{5mm}		\Esp_\nu[\abs{\Lambda_\epsilon}] \leq 23d^2\abs{\cA}\dfrac{\log(1/\epsilon)}{\epsilon^6} \,,\]
 where $d$ is the maximum degree of nodes in $G$.
 \label{lem:ce_and_lambda_are_finite}
\end{lemma}
Thus combining them with \eqref{eq:decomp_T_epsilon} we obtain
\beqan
\Esp_\nu[\abs{\cT_\epsilon^c}] 
&\leq& \Esp_\nu[\abs{\cE_\epsilon^c}] + \Esp_\nu[\abs{\Lambda_\epsilon}] \\
&\leq& \dfrac{10\abs{\cA}^2}{\epsilon^4} + 23d^2\abs{\cA}\dfrac{\log(1/\epsilon)}{\epsilon^6} \\
&\leq& 33d\abs{\cA}^2\dfrac{\log(1/\epsilon)}{\epsilon^6} \,.
\eeqan
Hence, we just proved the following lemma.
\begin{lemma}[Reliable estimators] \label{unimodal reliability}For $ 0 \!<\! \epsilon \!<\! \epsilon_\nu$,
\[
\Esp_\nu[\abs{\cT_\epsilon^{ c}}] \leq 33d\abs{\cA}^2\dfrac{\log(1/\epsilon)}{\epsilon^6} \,,
\]
where $d$ is the maximum degree of nodes in $G$.
\end{lemma}

	\subsection{\label{subsec: proof theorem}Upper bounds on the numbers of pulls of sub-optimal arms}
	In this section, we now combine the different results of the previous sections to prove Theorem~\ref{th:upper bounds}.

	\begin{proof}[Proof of Theorem~\ref{th:upper bounds}.] From Lemma~\ref{unimodal reliability}, considering the following subset of times
	\[
		\cT_\epsilon \coloneqq \left\{\begin{array}{l}
		    \hspace{-2mm} t \geq 1:\  \Ahat^\star(t) = \Set{a^\star}  \\
		     \hspace{10mm}\forall a \in \Set{ a^\star,a_{t+1} },\ \abs{\muhat_a(t) - \mu_a } < \epsilon 
		\end{array} \! \right\}\,. 
		\]
		we have 
		\[
		\Esp_\nu[\abs{\cT_\epsilon^{ c}}] \leq 33d\abs{\cA}^2\dfrac{\log(1/\epsilon)}{\epsilon^6} \,,
		\]
		$ \abs{\cA} \!=\! 11  $ $ \abs{\cA} \!=\! 10^2  $ $ \abs{\cA} \!=\! 10^3  $ $ \abs{\cA} \!=\! 10^4  $ where $d$ is the maximum degree of nodes in $G$.
		Then, let us consider $ a \!\neq\! a^\star$ and a time step $t \!\in\! \cT_\epsilon$ such that $a_{t+1} \!=\! a$. From Lemma~\ref{unimodal empirical upper bounds} we get
		\[
		N_a(t) \KLof{\muhat_a(t)}{\muhat^\star(t)} \leq \log(t) \leq \log(T) \,.
		\]
		Furthermore, since $ t \!\in\! \cT_\epsilon $, we have
		\[
		\ahat^\star_t = a^\star \ad \abs{\muhat_a(t) - \mu_a}, \abs{\muhat_{a^\star}(t) - \mu_{a^\star}} < \epsilon \,.
		\]
		According to the strategy $a \!=\! a_{t+1} \!\in\! \cV_{a^\star} $ and by construction of $\alpha_\nu(\cdot)$ (see Section~\ref{imed_unimodal notations} \!Notations)
		\[
		\KLof{\muhat_a(t)}{\muhat^\star(t)} = \KLof{\muhat_a(t)}{\muhat_{a^\star}(t)} \geq \dfrac{\KLof{\mu_a}{\mu_{a^\star}}}{1 + \alpha_\nu(\epsilon)} \] and \[N_{a}(t) \leq \dfrac{1 + \alpha_\nu(\epsilon)}{\KLof{\mu_a}{\mu_{a^\star}}} \log(T)  \,.
		\]
		
		~\\ Thus, we have shown that for $a \!\neq\! a^\star$,
		\[
		\forall t \in \cT_\epsilon \textnormal{ s.t. } a_{t+1} = a:\ a \in \cV_{a^\star} \] and \[  N_{a}(t) \leq \dfrac{1 + \alpha_\nu(\epsilon)}{\KLof{\mu_a}{\mu_{a^\star}}} \log(T) \,.
		\]
		This implies:
		\[
		N_{a}(T) \!\leq\! \left\{\hspace{-2mm}\begin{array}{ll}
		\dfrac{1 \!+\! \alpha_\nu(\epsilon)}{\KLof{\mu_a}{\mu_{a^\star}}} \log(T) \!+\! \abs{\cT_\epsilon^c} \!+\! 1 &\hspace{-3mm}, \textnormal{ if } a \in \cV_{a^\star} \\
		\abs{\cT_\epsilon^c} \!+\! 1 &\hspace{-3mm}, \textnormal{ otherwise}.
		\end{array}\right. 
		\]
		Averaging these inequalities allows us to conclude.
	\end{proof}

	\section{Numerical experiments}
	In this section, we consider Gaussian distributions with variance $\sigma^2 \!=\!1$ and compare empirically the following 
	strategies introduced beforehand:\OSUB described in 
	Algorithm~\ref{alg:osub}, \IMEDUB, \dIMEDUB described in 
	Algorithms~\ref{alg:imedub},\ref{alg:dimedub}, \KLUCBUB 
	described in Algorithm~\ref{alg:wucbub} as well as the 
	baseline \IMED by \citet{honda2011asymptotically} that does 
	not exploit the structure and finally the generic \OSSB  
	strategy by \citet{combes2017minimal} that adapts to several 
	structures. We compare these strategies on two setups.
	%, \textit{fixed configurations} and \textit{random configurations} described below, each with $|\cA|\! =\!17$ arms. Then we limit ourselves to \textit{random configurations} and compare these strategies for larger numbers of arms $\abs{\cA} \!\in\!\Set{10^2,10^3, 10^4}$. 

\paragraph{Fixed configuration}(Figure~\ref{fig:fixed configurations}). For the first experiments  we consider a small number of arms $\abs{\cA}\!=\!11$ and investigate these strategies over $500$ runs on  \emph{fixed} Gaussian configuration $\nu^0 \!\in\! \cD$ with means $\left(\mu^0_a\right)_{a \in \cA} \!=\! \left(0, 0.2, 0.4, 0.6, 0.8, 1, 0.8, 0.6, 0.4, 0.2, 0\right)$.

\paragraph{ Random configurations }(Figure~\ref{fig:random configurations} ). In this experiment we consider larger numbers of arms 
%$\abs{\cA} \!\in\!\Set{10^3, 10^4}$
$\abs{\cA} \!\in\!\Set{10^2,10^3, 10^4}$
and average regrets over $500$ random Gaussian configurations uniformly sampled in $\Set{\nu \!\in\! \cD\!: (\mu_a)_{a \in\cA}\in[0,1]^\cA}$. 

It seems that for a small number of arms \IMEDUB and \KLUCBUB perform better than the baseline \IMED whereas \OSSB performs very poorly for unimodal structure (this may be the price its genericity). Both \IMEDUB and \KLUCBUB outperform \OSUB significantly. When the set of arms becomes larger, only \dIMEDUB benefits from the unimodal structure and outperforms the baseline \IMED.
\begin{figure}[H]
    \centering
    \includegraphics[width=0.45\textwidth]{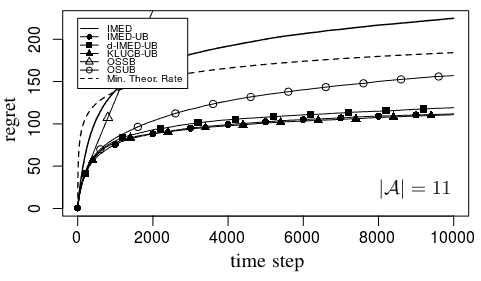}
  \caption{Regret approximated over $500$ runs for $\nu_0\!\in\!\cD$.}
 \label{fig:fixed configurations}
\end{figure}
\begin{remark}
  It is generally observed in bandit problems that theoretical asymptotic lower bounds on the regret are larger than the actual regret in finite horizon, as is it in Figure~\ref{fig:fixed configurations}.
\end{remark}
 %\hs{On peut faire sauter la remarque 4}
 %\pmd{ouais j'y ai pensé}
%  \pmd{bah les deux passent, il faut couper ailleurs}
% \pmd{mais ca ne va pas suffir il faut trouver autre chose pour tout faire passer}
\begin{figure}[H]
    \centering
    \includegraphics[width=0.45\textwidth]{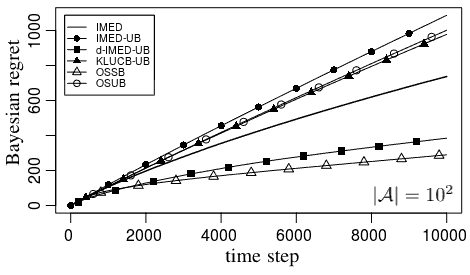}\\
      \includegraphics[width=0.45\textwidth]{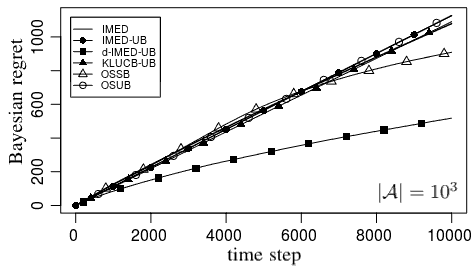}\\
      \includegraphics[width=0.45\textwidth]{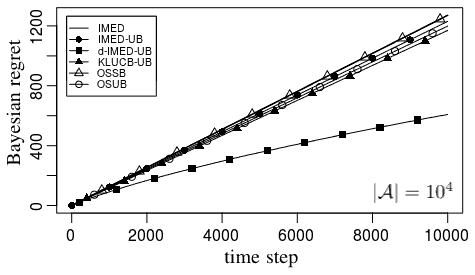}
  \caption{Regret averaged over $500$ random configurations in $\cD$.}
 \label{fig:random configurations}
\end{figure}
	%\oam{ Experiments? }
	%{\noindent \em Remainder omitted in this sample. See http://www.jmlr.org/papers/ for full paper.}
	
	% Acknowledgements should go at the end, before appendices and references

	\section*{Conclusion}
	
	In this paper, we have revisited the setup of unimodal multi-armed bandits: We introduced three novel variants, two based on the \IMED strategy and a second one using a \KLUCB type index but modified using tools similar to \IMED. These strategies do not require forcing to play specific arms (unlike for instance \OSUB) on top of the naturally introduced score. Remarkably, the \IMEDUB and \dIMEDUB strategies  do not require any optimization procedure, which can be interesting for practitioners. We also provided a novel proof strategy (inspired from \IMED), in which we make explicit empirical lower and upper bounds, before tackling the handling of bad events by more standard concentration tools. This proof technique greatly simplifies and shorten the analysis of \IMEDUB (compared to that of \OSUB), and is also employed to analyze  \KLUCBUB and \dIMEDUB, in a somewhat unified way. 
	Last, we provided numerical experiments that show the practical advantages of the novel approach over the \OSUB strategy.

%\acks{ \dots }


\newpage

\appendix

\section{\IMEDUB finite time analysis}
	\label{app: imed_analysis}
	We regroup in this section, for completeness, the proofs of the remaining lemmas used in the analysis of \IMEDUB in Section~\ref{sec : imed_analysis}.
	
	\subsection{Concentration lemmas}
	We state two concentration lemmas that do not depend on the followed strategy. Lemma~\ref{unimodal concentration} comes from Lemma B.1 in \citet{combes2014unimodal} and Lemma~\ref{unimodal large deviation} comes from Lemma 14 in \citet{honda2015imed}. Proofs are provided in Appendix~\ref{app: concentration_lemmas}.
	\label{subsec : imed_concentration}
	\begin{lemma}[Concentration inequalities]\label{unimodal concentration} Independently of the considered strategy, for all set of Gaussian or Bernoulli distributions $\nu \!\in\! \cD $, for all $  0 \!<\! \epsilon \!\leq\! 1/2 $, for all $a, a' \!\in\! \cA$, we have
		\[
		\Esp_\nu\!\left[\sum\limits_{t \geq 1}\ind_{\Set{a_{t+1}=a,\ N_{a'}(t) \geq   N_{a}(t),\  \abs{\muhat_{a'}(t) - \mu_{a'}} \geq \epsilon }} \right] \leq \dfrac{10}{\epsilon^4} \,.
		\]
	\end{lemma}

	\begin{lemma}[Large deviation probabilities]\label{unimodal large deviation} Let us consider a set of Gaussian or Bernoulli distributions $\nu \!\in\! \cD $. Let $  0 \!<\! \epsilon \!\leq\! \min(1\!-\!\mu^\star,1/2) $ and $a \!\in\! \cA$. Let $\lambda \!=\! \mu_a \!-\! \epsilon$. Then, independently of the considered strategy, we have
		\[ 
		\Esp_\nu\!\left[\sum\limits_{n \geq 1}\ind_{\Set{ \muhat_{a}^{n}  < \lambda}}n\exp\!\left(n \KLof{\muhat_{a}^{n}}{\lambda}\right)\right] \leq \dfrac{23\log(1/\epsilon)}{\epsilon^6}  \,.
		\]
	\end{lemma}
	
\subsection{Proof of Lemma~\ref{lem:ce_and_lambda_are_finite} (Bounded subsets of times)}
\label{app:proof_ce_and_lambda_finite}
%\pmd{TODO clean this proof}
% \textbf{Lemma} Conserving the same notations as in Section~\ref{sec : imed_analysis}:
% 	\[
% 	\Esp_\nu[\abs{\Lambda_\epsilon}] \leq \dfrac{48\abs{\cA}}{\epsilon_\nu^4} \,.
% 	\]
Using Lemma~\ref{unimodal empirical lower bounds} we have
		\[
		\forall t \geq 1, \quad N_{a_{t+1}}(t) \leq N_{\ahat^\star_t}(t) \,.
		\]
		Since $ \ahat^\star_t \!\in\!\argmin\limits_{\ahat^\star \in \Ahat^\star(t)}N_{\ahat^\star}(t)$, this implies
		\[
			\forall t \geq 1,\forall \ahat^\star \in \Ahat^\star(t),  \quad N_{a_{t+1}}(t) \leq N_{\ahat^\star_t}(t) \leq N_{\ahat^\star}(t) \,.
		\]
		Then, based on the concentration inequalities from Lemma~\ref{unimodal concentration},  we obtain
		\beqan 
		\Esp_\nu[\abs{\cE_\epsilon^c}] 
		&\leq& \sum\limits_{a, a' \in \cA} \Esp_\nu\!\left[\sum\limits_{t\geq1}{\ind_{\Set{a_{t+1}=a,\ N_{a'}(t) \geq  N_{a}(t),\  \abs{\muhat_{a'}(t) - \mu_{a'}} \geq \epsilon }}}\right] \\
		&\leq & \sum\limits_{a, a' \in \cA} \dfrac{10}{\epsilon^4} \\
		&\leq&  \dfrac{10\abs{\cA}^2}{\epsilon^4} \,.
		\eeqan
Furthermore, for $t \!\geq\! 1$ and $a \!\in\! \cA$, we have
		{\small
			\[
			\log\!\left(N_{a_{t+1}}(t)\right) \leq  N_a(t) \KLof{\muhat_a(t)}{\lambda_a}  + \log\!\left(N_{a}(t)\right) \quad \Leftrightarrow \quad  N_{a_{t+1}}(t)  \leq   N_a(t) \exp\!\left(N_a(t) \KLof{\muhat_a(t)}{\lambda_a}\right) \,,
			\]
		}
		~\\where $\lambda_a \!=\! \mu_a \!-\! \epsilon $ for all arm $a \!\in\! \cA$.
		~\\Thus, we have
		{\footnotesize
			\beqan  \abs{\Lambda_\epsilon}
			&\leq&\sum\limits_{t \geq 1}\sum\limits_{ a \in \cV_{\ahat^\star_t}\!\setminus\!\Set{a_{t+1},\ahat^\star_t}}  \ind_{\Set{  \muhat_a(t) < \lambda_a \textnormal{ and } N_{a_{t+1}}(t) \leq  N_a(t) \exp\left(N_a(t) \KL(\muhat_a(t)|\lambda_a)\right)  }} \\
			&=& \sum\limits_{t \geq 1}\sum\limits_{\ahat^\star\in\cA}\sum\limits_{ a' \in \Set{\ahat^\star}\cup\cV_{\ahat^\star}}\sum\limits_{ a \in \cV_{\ahat^\star}\!\setminus\!\Set{a',\ahat^\star}}\sum\limits_{n \geq 1} \ind_{\Set{ \ahat^\star_t = \ahat^\star, a_{t+1} = a', N_a(t) = n }} \ind_{\Set{  \muhat_a^n < \lambda_a,\ N_{a'}(t) \leq n \exp\left(n \KL(\muhat_a^n|\lambda_a)\right)  }} \\
			&\leq& \sum\limits_{t \geq 1}\sum\limits_{\ahat^\star\in\cA}\sum\limits_{ a' \in \Set{\ahat^\star}\cup\cV_{\ahat^\star}}\sum\limits_{ a \in \cV_{\ahat^\star}\!\setminus\!\Set{a',\ahat^\star}}\sum\limits_{n \geq 1} \ind_{\Set{ a_{t+1} = a' }} \ind_{\Set{  \muhat_a^n < \lambda_a }} \ind_{\Set{ N_{a'}(t) \leq n \exp\left(n \KL(\muhat_a^n|\lambda_a)\right)  }} \\
			&=& \sum\limits_{\ahat^\star\in\cA}\sum\limits_{ a' \in \Set{\ahat^\star}\cup\cV_{\ahat^\star}}\sum\limits_{ a \in \cV_{\ahat^\star}\!\setminus\!\Set{a',\ahat^\star}}\sum\limits_{n \geq 1}\ind_{\Set{  \muhat_a^n < \lambda_a }} \sum\limits_{t \geq 1}   \ind_{\Set{ a_{t+1} = a' \textnormal{ and } N_{a'}(t) \leq n \exp\left(n \KL(\muhat_a^n|\lambda_a)\right)  }} \\
			&\leq& \sum\limits_{\ahat^\star\in\cA}\sum\limits_{ a' \in \Set{\ahat^\star}\cup\cV_{\ahat^\star}}\sum\limits_{ a \in \cV_{\ahat^\star}\!\setminus\!\Set{a',\ahat^\star}}\sum\limits_{n \geq 1}\ind_{\Set{  \muhat_a^n < \lambda_a }} n \exp\!\left(n \KLof{\muhat_a^n}{\lambda_a}\right) \\
			&\leq&  \sum\limits_{\ahat^\star\in\cA}\sum\limits_{ a' \in \Set{\ahat^\star}\cup\cV_{\ahat^\star}}\sum\limits_{ a \in \cV_{\ahat^\star}\!\setminus\!\Set{a',\ahat^\star}}\sum\limits_{n \geq 1}\ind_{\Set{  \muhat_a^n < \lambda_a }} n \exp\!\left(n \KLof{\muhat_a^n}{\lambda_a}\right) 
			\eeqan
		}
		and
		\[
		\Esp_\nu[\abs{\Lambda_\epsilon}] \leq  \sum\limits_{\ahat^\star\in\cA}\sum\limits_{ a' \in \Set{\ahat^\star}\cup\cV_{\ahat^\star}}\sum\limits_{ a \in \cV_{\ahat^\star}\!\setminus\!\Set{a',\ahat^\star}} \Esp_\nu\!\left[ \sum\limits_{n \geq 1}  \ind_{\Set{  \muhat_a^n < \lambda_a }}  n \exp\!\left(n \KLof{\muhat_a^n}{\lambda_a}\right) \right] \,.
		\]
		Then, by applying Lemma~\ref{unimodal large deviation} based on large deviation probabilities, we have
		\[
		\forall a \in \cA,\quad \Esp_\nu\!\left[ \sum\limits_{n \geq 1}  \ind_{\Set{  \muhat_a^n < \lambda_a }}  n \exp\left(n \KL(\muhat_a^n|\lambda_a)\right) \right] \leq \dfrac{23\log(1/\epsilon)}{\epsilon^6} \,.
		\]
		It comes:
		\[
		\Esp_\nu[\abs{\Lambda_\epsilon}] \leq  \sum\limits_{\ahat^\star\in\cA}\sum\limits_{ a' \in \Set{\ahat^\star}\cup\cV_{\ahat^\star}}\sum\limits_{ a \in \cV_{\ahat^\star}\!\setminus\!\Set{a',\ahat^\star}} \dfrac{23\log(1/\epsilon)}{\epsilon^6}  \leq 23d^2\abs{\cA}\dfrac{\log(1/\epsilon)}{\epsilon^6}\,.
		\]

	\section{Concentration lemmas}
	\label{app: concentration_lemmas}
	{\bf Lemma} Independently of the considered strategy, for all set of Gaussian or Bernoulli distributions $\nu \!\in\! \cD $, for all $  0 \!<\! \epsilon \!\leq\! 1/2 $, for all $a, a' \!\in\! \cA$, we have
	\[
	\Esp_\nu\left[\sum\limits_{t \geq 1}\ind_{\Set{a_{t+1}=a,\ N_{a'}(t) \geq   N_{a}(t),\  \abs{\muhat_{a'}(t) - \mu_{a'}} \geq \epsilon }} \right] \leq \dfrac{10}{\epsilon^4} \,.
	\]
	\begin{proof} Considering the stopping times $ \tau_{a,n} \!=\! \inf{ \Set{t \!\geq\! 1, N_{a}(t) \!=\! n} }$  we will rewrite the sum  
		~\\$ \sum\limits_{t \geq 1}\ind_{\Set{a_{t+1}=a,\ N_{a'}(t) \geq  N_{a}(t),\  \abs{\muhat_{a'}(t) - \mu_{a'}} \geq \epsilon }} $ and use an Hoeffding's type argument for distributions with support included in $[0,1]$.
		\beqan
		&& \sum\limits_{t \geq 1}\ind_{\Set{a_{t+1}=a,\ N_{a'}(t) \geq  N_{a}(t),\  \abs{\muhat_{a'}(t) - \mu_{a'}} \geq \epsilon }} \\
		&=& \sum\limits_{t \geq 1}\sum\limits_{n \geq 2, m \geq 1}\ind_{\Set{ \tau_{a,n} = t +1, N_{a'}(t) = m }} \ind_{\Set{ m \geq  n-1 ,\  \abs{\muhat_{a'}^m - \mu_{a'}} \geq \epsilon }} \\
		&=& \sum\limits_{ m \geq 1} \sum\limits_{n \geq 2}\ind_{\Set{ m \geq n-1,\  \abs{\muhat_{a'}^m - \mu_{a'}} \geq \epsilon }}\sum\limits_{t \geq 1}\ind_{\Set{ \tau_{a,n} = t +1, N_{a'}(t) = m }}  \\
		&\leq& \sum\limits_{ m \geq 1} \sum\limits_{n \geq 2}\ind_{\Set{ m \geq n-1,\  \abs{\muhat_{a'}^m - \mu_{a'}} \geq \epsilon }}\sum\limits_{t \geq 1}\ind_{\Set{ \tau_{a,n} = t +1}}  \\
		&\leq& \sum\limits_{ m \geq 1} \sum\limits_{n \geq 2}\ind_{\Set{ m \geq  n-1,\  \abs{\muhat_{a'}^m - \mu_{a'}} \geq \epsilon }} \\
		\eeqan 
		Taking the expectation , it comes
		\beqan
		&&\Esp_\nu\left[\sum\limits_{t \geq 1}\ind_{\Set{a_{t+1}=a,\ N_{a'}(t) \geq   N_{a}(t),\  \abs{\muhat_{a'}(t) - \mu_{a'}} \geq \epsilon }}\right]  \\
		&\leq& \sum\limits_{ m \geq 1} \sum\limits_{n \geq 2}\ind_{\Set{ m \geq n-1 }} \Pr_\nu\left(  \abs{\muhat_{a'}^m - \mu_{a'}} \geq \epsilon  \right) \\
		&\leq& \sum\limits_{ m \geq 1} \sum\limits_{n \geq 2}\ind_{\Set{ m \geq n-1 }} \max\!\left(2\e^{-2m\epsilon^2},\ 2\e^{-m\epsilon^2/2} \right) \qquad \textnormal{(Hoeffding's inequality)}\\
		&=& \sum\limits_{ m \geq 1} \sum\limits_{n \geq 2}\ind_{\Set{ m \geq n-1 }} 2\e^{-m\epsilon^2/2} \\
		&=& 2\sum\limits_{ m \geq 1} m \e^{-m\epsilon^2/2} \\
		&=&  \dfrac{2\e^{-\epsilon^2/2}}{(1 - \e^{-\epsilon^2/2})^2} =  \dfrac{2\e^{\epsilon^2/2}}{(\e^{\epsilon^2/2} - 1)^2} \leq \dfrac{8\e^{1/8}}{ \epsilon^4} \leq \dfrac{10}{\epsilon^4} \,. \qquad \textnormal{(}0 < \epsilon \leq 1/2\textnormal{)} 
		\eeqan
		
	\end{proof}
	
	~\\ {\bf Lemma} Let us consider a set of Gaussian or Bernoulli distributions $\nu \!\in\! \cD $. Let $  0 \!<\! \epsilon \!\leq\! \min(1\!-\!\mu^\star,1/2) $ and $a \!\in\! \cA$. Let $\lambda \!=\! \mu_a \!-\! \epsilon$. Then, independently of the considered strategy, we have
	\[ 
	\Esp_\nu\left[\sum\limits_{n \geq 1}\ind_{\Set{ \muhat_{a}^{n}  < \lambda}}n\exp(n \KLof{\muhat_{a}^{n}}{\lambda})\right]  \leq \dfrac{23\log(1/\epsilon)}{\epsilon^6} \,.
	\]
	We provide two proofs, one for Gaussian distributions and another for Bernoulli distributions, that can be read separately. 
	\begin{proof}[For Gaussian distributions] The proof is based on a Chernoff type inequality and a calculation by measurement change.
	
	Since $ \nu_a \sim \cN(\mu_a,1) $ we have for all $ \epsilon > 0$,
		\[	\forall n \geq 1,\quad  \Pr_\nu(\muhat_a^n - \mu_a \leq - \epsilon) \leq \e^{-n \epsilon^2/2} \,. 
		\]
		In addition, $\forall \mu,\mu' \!\in\! \Real,\quad \KL(\mu|\mu') \!=\! \dfrac{(\mu \!-\! \mu')^2}{2}$.
		Let $ n \geq 1 $. We have:
		{\small
			\beqan
			&&\Esp_\nu\left[\ind_{\Set{\muhat_a^n \leq \lambda}} n \e^{n \KL(\muhat_a^n|\lambda)}\right] \\
			&=& \int_0^\infty \Pr_\nu\left( \ind_{\Set{\muhat_a^n \leq \lambda}} n \e^{n \KL(\muhat_a^n|\lambda)} > x \right) \mathrm{d} x \\
			&=& \int_0^\infty \Pr_\nu\left(  n \e^{n \KL(\muhat_a^n|\lambda)} > x,\ \muhat_a^n \leq \lambda \right) \mathrm{d}x \\
			&=& \int_{-\infty}^\infty n^2 \e^{n u} \Pr_\nu\left(   \KL(\muhat_a^n|\lambda) > u,\ \muhat_a^n \leq \lambda \right) \mathrm{d}u \qquad (\textnormal{variable change } x = n \e^{n u} ,\ \mathrm{d}x = n^2 \e^{n u} \mathrm{d}u) \\
			&=& \int_{-\infty}^{0} n^2 \e^{n u} \Pr_\nu\left(\KL(\muhat_a^n|\lambda) > u,\  \muhat_a^n \leq \lambda\right) \mathrm{d}u + \int_{0}^\infty n^2 \e^{n u} \Pr_\nu\left( \KL(\muhat_a^n|\lambda) > u,\ \muhat_a^n \leq \lambda \right) \mathrm{d}u \\
			&=&  n^2  \Pr_\nu\left(\muhat_a^n - \mu_a \leq - \epsilon \right) \int_{-\infty}^{0}\e^{n u} \mathrm{d}u + \int_{0}^\infty n^2 \e^{n u} \Pr_\nu\left( \muhat_a^n - \mu_a \leq - \epsilon - \sqrt{2u} \right) \mathrm{d}u \\
			&\leq&  n^2  \e^{-n\epsilon^2/2}
			\dfrac{1}{n} + \int_{0}^\infty n^2 \e^{n u} \e^{-n (\epsilon + \sqrt{2u})^2/2 } \mathrm{d}u \\
			&=&  n  \e^{-n\epsilon^2/2}
			+ n^2 \e^{-n\epsilon^2/2} \int_{0}^\infty   \e^{- n \epsilon \sqrt{2u} } \mathrm{d}u \\
			&=&  n  \e^{-n\epsilon^2/2}
			+ n^2 \e^{-n\epsilon^2/2} \int_{0}^\infty y  \e^{- n \epsilon y} \mathrm{d}y \qquad (\textnormal{variable change } u = \dfrac{y^2}{2} ,\ \mathrm{d}u = y \mathrm{d}y)\\
			&=&  n  \e^{-n\epsilon^2/2}
			+ n^2 \e^{-n\epsilon^2/2} \dfrac{1}{(n\epsilon)^2}\\
			&=& n \e^{-n\epsilon^2/2} + \dfrac{1}{\epsilon^2} \e^{-n\epsilon^2/2} 
			\eeqan
		}
		To ends the proof, we use the following equalities for $ r > 0 $
		\beqan
		&& \sum\limits_{n\geq 1}\e^{-n r} = \dfrac{\e^{-r}}{1-\e^{-r}}  \\
		&& \sum\limits_{n\geq 1}n\e^{-n r} = \dfrac{\e^{-r}}{(1-\e^{-r})^2} 
		\eeqan
	and obtain
	\[ 
	\Esp_\nu\left[\sum\limits_{n \geq 1}\ind_{\Set{ \muhat_{a}^{n}  < \lambda}}n\exp(n \KL(\muhat_{a}^{n}|\lambda))\right]  \leq \dfrac{1}{\epsilon^2}\dfrac{\e^{-\epsilon^2/2}}{1-\e^{-\epsilon^2/2}} + \dfrac{\e^{-\epsilon^2/2}}{(1-\e^{-\epsilon^2/2})^2} \leq \dfrac{10}{\epsilon^4} \,.
	\]	
	\end{proof}
	
	\begin{proof}[For Bernoulli distributions] The proof is based on a Chernoff type inequality and a calculation by measurement change.
	
	Since the support of $ \nu_a$ is included in $[0,1]$  we have by Chernoff's and Pinsker's inequalities
		\[
		\forall 0\leq v \leq \mu_a, \forall n \geq 1,  \quad  \Pr_\nu(\muhat_a^n \leq v) \leq \e^{-n\kl(v|\mu_a)} \leq \e^{-2n (\mu_a-v)^2} \,.
		\]
		Let $ n \geq 1 $. We have
		
			\beqan
			&&\Esp_\nu\left[\ind_{\Set{\muhat_a^n \leq \lambda}} n \e^{n \kl(\muhat_a^n|\lambda)}\right] \\
			&=& \int_0^\infty \Pr_\nu\left( \ind_{\Set{\muhat_a^n \leq \lambda}} n \e^{n \kl(\muhat_a^n|\lambda)} > x \right) \mathrm{d}x \\
			&=& \int_0^\infty \Pr_\nu\left(  n \e^{n \kl(\muhat_a^n|\lambda)} > x,\ \muhat_a^n \leq \lambda \right) \mathrm{d}x \\
			&=& \int_{-\infty}^\infty n^2 \e^{n u} \Pr_\nu\left(   \klof{\muhat_a^n}{\lambda} > u,\ \muhat_a^n \leq \lambda \right) \mathrm{d}u \qquad (\textnormal{variable change } x = n \e^{ n u} ,\ \mathrm{d}x = n^2 \e^{n u} \mathrm{d}u) \\
			&=& \int_{-\infty}^{0} n^2 \e^{n u} \Pr_\nu\left(\klof{\muhat_a^n}{\lambda} > u,\  \muhat_a^n \leq \lambda\right) \mathrm{d}u + \int_{0}^{\kl(0|\lambda)} n^2 \e^{n u} \Pr_\nu\left( \klof{\muhat_a^n}{\lambda} > u,\ \muhat_a^n \leq \lambda \right) \mathrm{d}u \,.
			\eeqan
			On the one hand 
			\beqan
			&& \int_{-\infty}^{0} n^2 \e^{n u} \Pr_\nu\left(\klof{\muhat_a^n}{\lambda} > u,\  \muhat_a^n \leq \lambda\right) \mathrm{d}u \\
			&=&  n^2  \Pr_\nu\left(\muhat_a^n  \leq \lambda \right) \int_{-\infty}^{0}\e^{n u} \mathrm{d}u \\ 
			&\leq&  n^2  \e^{-2n(\mu_a -\lambda)^2}
			\dfrac{1}{n}  \\
			&=& n \e^{-2n\epsilon^2} \,.
			\eeqan
			On the other hand, using variable change $u \!=\! \kl(v|\lambda)$ and Lemma~\ref{kl_inequality}, it comes
			\beqan
			&& \int_{0}^{\kl(0|\lambda)} n^2 \e^{n u} \Pr_\nu\left( \klof{\muhat_a^n}{\lambda} > u,\ \muhat_a^n \leq \lambda \right) \mathrm{d}u \\
			&=&  \int_{0}^{\lambda} n^2 \e^{n \kl(v|\lambda)} \Pr_\nu\left( \klof{\muhat_a^n}{\lambda} > \klof{v}{\lambda},\ \muhat_a^n \leq \lambda \right)  \left[-\dfrac{\partial\kl}{\partial p}(v|\lambda) \right]\mathrm{d}v  \\
			&=&  \int_{0}^{\lambda} n^2 \e^{n \kl(v|\lambda)} \Pr_\nu\left( \muhat_a^n< v \right)  \left[-\dfrac{\partial\kl}{\partial p}(v|\lambda) \right]\mathrm{d}v  \\
			&\leq&  \int_{0}^{\lambda} n^2 \e^{n \kl(v|\lambda)} \e^{- n \kl(v|\mu_a)} \left[-\dfrac{\partial\kl}{\partial p}(v|\lambda) \right]\mathrm{d}v  \\
			&=&  \int_{0}^{\lambda} n^2 \e^{- n \left(\kl(v|\mu_a) - \kl(v|\lambda)\right)} \left[-\dfrac{\partial\kl}{\partial p}(v|\lambda) \right]\mathrm{d}v  \\
			&\leq&  \int_{0}^{\lambda} n^2 \e^{- n \frac{(\mu_a - \lambda)^2}{2}} \left[-\dfrac{\partial\kl}{\partial p}(v|\lambda) \right]\mathrm{d}v  \qquad \textnormal{( Lemma~\ref{kl_inequality})} \\
			&=&   \kl(0|\lambda)n^2 \e^{- n\epsilon^2/2}  \\
			&\leq&  \left(-\log(1-\mu^\star)\right) n^2 \e^{- n\epsilon^2/2}  \\
			&\leq&  \log(1/\epsilon) n^2 \e^{- n\epsilon^2/2} \,,
			\eeqan
		where $\dfrac{\partial\kl}{\partial p}$ corresponds to the derivative of the $\kl(\cdot|\cdot)$ according to the first variable.
		Thus we have
		\[
		\Esp_\nu\left[\ind_{\Set{\muhat_a^n \leq \lambda}} n \e^{n \kl(\muhat_a^n|\lambda)}\right] \leq n \e^{-2n\epsilon^2} +  \log(1/\epsilon) n^2 \e^{- n\epsilon^2/2} \,. 
		\]
		To ends the proof, we use the following equalities for $ r > 0 $
		\beqan
		&& \sum\limits_{n\geq 1}\e^{-n r} = \dfrac{\e^{-r}}{1 - \e^{-r}} = \dfrac{1}{\e^r - 1} \\
		&& \sum\limits_{n\geq 1}n\e^{-n r} = \dfrac{\e^r}{(\e^r - 1)^2} = \dfrac{1}{\e^r - 1} +\dfrac{1}{(\e^r - 1)^2}\\
		&& \sum\limits_{n\geq 1}n^2\e^{-n r} = \dfrac{\e^r}{(\e^r - 1)^2} + \dfrac{2\e^r}{(\e^r - 1)^3} = \dfrac{\e^{2r} + \e^r}{(\e^r - 1)^3} 
		\eeqan
		and obtain 
			\[ 
	\Esp_\nu\left[\sum\limits_{n \geq 1}\ind_{\Set{ \muhat_{a}^{n}  < \lambda}}n\exp(n \klof{\muhat_{a}^{n}}{\lambda})\right]  \leq \dfrac{\e^{2\epsilon^2}}{(\e^{2\epsilon^2} - 1)^2} +  \dfrac{\log(1/\epsilon)\e^{\epsilon^2} + \e^{\epsilon^2/2}}{(\e^{\epsilon^2/2} - 1)^3} \leq \dfrac{23\log(1/\epsilon)}{\epsilon^6} \,.
	\] 
	\end{proof}
	\begin{lemma}\label{kl_inequality}For all $0\!\leq\!v\!\leq\lambda\!<\!\mu \!<\! 1$ we have 
	\[
	\klof{v}{\mu} - \klof{v}{\lambda} \geq \dfrac{(\mu - \lambda)^2 }{2} \,.
	\]
	\end{lemma}
	\begin{proof}
	Using monotony of the $\kl(\cdot|\cdot)$ we get 
	\[
	\klof{v}{\mu} - \klof{v}{\lambda} \geq \klof{v}{\mu} -  \kl\!\left(v\left|\frac{\mu+\lambda}{2}\right. \right) \,.
	\]
	Using convexity of the $\kl(\cdot|\cdot)$ we get
	\[
	\dfrac{\klof{v}{\mu} -  \kl\!\left(v\left|\frac{\mu+\lambda}{2}\right.\right) }{ \mu - \frac{\mu+\lambda}{2}} \geq \dfrac{\partial\kl}{\partial q}\!\left(v\left|\frac{\mu+\lambda}{2}\right. \right) \geq \dfrac{\partial\kl}{\partial q}\!\left(\lambda \left|\frac{\mu+\lambda}{2}\right. \right) \,,
	\]
	where $\dfrac{\partial\kl}{\partial q}$ corresponds to the derivative of the $\kl(\cdot|\cdot)$ according to the second variable. From Lemma B.4 in \citet{combes2014unimodal} we have 
	\[ \left(\frac{\mu+\lambda}{2} - \lambda \right)
	\dfrac{\partial\kl}{\partial q}\!\left(\lambda \left|\frac{\mu+\lambda}{2}\right. \right) \geq \kl\!\left(\lambda\left|\frac{\mu+\lambda}{2}\right. \right) \,.
	\]
	 Then Pinsker's inequality implies 
	\[
	\left(\frac{\mu+\lambda}{2} - \lambda \right)
	\dfrac{\partial\kl}{\partial q}\!\left(\lambda \left|\frac{\mu+\lambda}{2}\right. \right) \geq 2 \left(\frac{\mu+\lambda}{2} - \lambda \right)^2 = \dfrac{(\mu - \lambda)^2 }{2} \,,
	\]
	which ends the proof.
	\end{proof}
	
	\section{\KLUCBUB  finite time analysis}
	\label{app: ucb_analysis}
	\KLUCBUB strategy implies similar lower bounds and empirical upper bounds on the numbers of pulls as \IMEDUB strategy.  An additional random process $(\gamma_t)_{t \geq 1} \!\in\! \Set{0,1}$ appears in the empirical lower bounds induced by  \KLUCBUB strategy (Lemma~\ref{ucb_unimodal empirical lower bounds}). When $\gamma_t \!=\! 1$, the empirical bounds are the same of the ones induced by \IMEDUB strategy. And we show that the process $(\gamma_t)_{t\geq 1}$ reaches zero only a finite number of times for which a use of the empirical bounds is needed.
	Then, similar reasoning as the one developed in Section~\ref{sec : imed_analysis} can be re-used and gives similar finite time analysis.
	
	%\oam{Quickly explain what are the key differences as well ? 
	%	Say We replace the result of Lemma bla with ..., etc.}
	%	\oam{It seems the main technical part of the proof comapred to that of \imedub\ is to provide the empirical lower and upper bounds.
	%	Then,  on top of the other similar terms, we have 1 additional bad event to control.}
	
	\subsection{Notations}
	Please, refer to Section~\ref{imed_unimodal notations}.
	
	\subsection{Strategy-based empirical bounds}
	In this subsection, we provide empirical bounds very similar to the ones induced by \IMEDUB strategy. We first establish preliminary results on the indexes. 
	
	~\\It is noticeable that  for all time step $ t \!\geq \!1$,
	\beq \label{U_astar} \forall \ahat^\star \in \Ahat^\star(t),\quad  U_{\ahat^\star}(t) = \muhat_{\ahat^\star}(t) = \muhat^\star(t)\,. \eeq
	In addition for $a \!\notin\! \Ahat^\star(t)$,
	\beqa
	&& \label{ucb_unimodal U_a = muhat_a} \textnormal{if } N_a(t) \geq N_{\ahat^\star_t}(t)\,,\quad U_a(t) = \muhat_a(t) \ad   U_a(t) < U_{\ahat^\star_t}(t)\,,  \\
	&& \label{ucb_unimodal U_a > muhat_a} \textnormal{if } N_a(t) < N_{\ahat^\star_t}(t)\,, \quad N_a(t) \KLof{\muhat_a(t)}{U_a(t)} + \log\!\left(N_a(t)\right) = \log\!\left(N_{\ahat^\star_t}(t)\right)\,.
	\eeqa
	In particular we have
	\beqa 
	&& \label{ucb_unimodal Nat+1 < Na*}  N_{a_{t+1}}(t) \leq N_{\ahat^\star_t}(t) \\ 
	&\ad& \label{ucb_unimodal index Uat+1} N_{a_{t+1}}(t) \KLof{\muhat_{a_{t+1}}(t)}{U_{a_{t+1}}(t)} + \log\!\left(N_{a_{t+1}}(t)\right) = \log\!\left(N_{\ahat^\star_t}(t)\right) \,.  
	\eeqa

	\begin{lemma}[Empirical lower bounds]\label{ucb_unimodal empirical lower bounds}Under \KLUCBUB, at each step time $t \!\geq \!1$, 
		\beqan 
		&1.& \forall a \in \cV_{\ahat^\star_t},\quad  \gamma_t\log\!\left(N_{a_{t+1}}(t)\right)  \leq N_{a}(t) \KLof{\muhat_{a}(t)}{\muhat^\star(t)} + \log\!\left((N_{a}(t)\right) \\
		&2.&  N_{a_{t+1}}(t) \leq N_{\ahat^\star_t}(t)\,,
		\eeqan 
		where $\gamma_t \!=\! \ind_{\Set{a_{t+1} \in \Ahat^\star(t) }} \!+\! \ind_{\Set{a_{t+1} \notin \Ahat^\star(t) \ad  \log(N_{a_{t+1}}(t)) \leq N_{a_{t+1}}(t) \KL(\muhat_{a_{t+1}}(t)|\muhat^\star(t))  }} \!\in\! \Set{0,1}$.
	\end{lemma}
	
	\begin{proof}  Let $t \!\geq \!1$. We have already seen that  $ N_{a_{t+1}}(t) \!\leq\! N_{\ahat^\star_t}(t)$ in (\ref{ucb_unimodal Nat+1 < Na*}). This corresponds to point 2. In the following, we prove point 1.
		
		~\\ For $a \!\in\! \cV_{\ahat^\star_t}$ such that $a \!=\! a_{t+1}$ or $N_a(t) \!\geq\! N_{a_{t+1}}(t)$, point 1. is naturally satisfied.
		
		~\\Let $a \!\in\! \cV_{\ahat^\star_t}$ such that $a \!\neq\! a_{t+1}$ and $N_a(t) \!<\! N_{a_{t+1}}(t)$. 
		
		~\\ \underline{\textbf{Case 1}} : $a_{t+1} \!=\! \ahat^\star_t$  
		~\\ Then $N_a(t) \!<\! N_{\ahat^\star_t}(t) \leq N_{a_{t+1}}(t)$ and from equations~(\ref{ucb_unimodal U_a > muhat_a})~and~(\ref{ucb_unimodal Nat+1 < Na*}) it comes:
		\[
		\log\!\left(N_{a_{t+1}}(t)\right) = \log\!\left(N_{\ahat^\star_t}(t)\right) \qquad
		N_a(t) \KLof{\muhat_a(t)}{U_a(t)} + \log\!\left(N_a(t)\right) = \log\!\left(N_{a_{t+1}}(t)\right)\,.
		\]
		According to the followed strategy and equation~\ref{U_astar}
		\[
		\muhat_a(t) \leq U_a(t) \leq U_{a_{t+1}}(t) \ad U_{a_{t+1}}(t) =  \muhat^\star(t) \,.
		\]
		Since $  a_{t+1} = \ahat^\star_t$, this implies 
		\[
		\muhat_a(t) \leq U_a(t) \leq \muhat^\star(t)  \,.
		\]
		Then the monotony of the $\KL$ implies 
		\[
		\KLof{\muhat_a(t)}{U_a(t)} \leq \KLof{\muhat_a(t)}{\muhat^\star(t)}
		\]
		and
		\beqan \log\!\left(N_{a_{t+1}}(t)\right)  &=&  N_a(t) \KLof{\muhat_a(t)}{U_a(t)} + \log\!\left(N_a(t)\right) \\
		&\leq&  N_a(t) \KLof{\muhat_a(t)}{\muhat^\star(t)} + \log\!\left(N_a(t)\right) \,.
		\eeqan

		~\\ \underline{\textbf{Case 2} } : $ a_{t+1} \!\neq\! \ahat^\star(t) $
		~\\ From equations~(\ref{ucb_unimodal U_a > muhat_a})  and  (\ref{ucb_unimodal index Uat+1}), we get
		\beq \label{ucb_unimodal indexes Uat+1 and Ua}N_{a_{t+1}}(t)\KLof{\muhat_{a_{t+1}}(t)}{U_{a_{t+1}}(t)} + \log\!\left(N_{a_{t+1}}(t)\right) = N_a(t) \KLof{\muhat_a(t)}{U_a(t)} + \log\!\left(N_a(t)\right) \,. \eeq
		Since $N_a(t) \!<\! N_{a_{t+1}}(t)$, this implies 
		\[
		\KLof{\muhat_{a_{t+1}}(t)}{U_{a_{t+1}}(t)} <  \KLof{\muhat_a(t)}{U_a(t)} \,.
		\]
		~\\ According to the followed strategy we have $ \muhat_a(t) \!\leq\! U_a(t) \!\leq\! U_{a_{t+1}}(t) $. 
		Then, the monotony of the $\KL$ implies 
		\[
		\KLof{\muhat_a(t)}{U_a(t)} \leq \KLof{\muhat_a(t)}{U_{a_{t+1}}(t)}   \,.
		\]
		Thus, we have:
		\[
		\KLof{\muhat_{a_{t+1}}(t)}{U_{a_{t+1}}(t)} < \KLof{\muhat_a(t)}{U_{a_{t+1}}(t)} \ad \muhat_a(t) \leq U_{a_{t+1}}(t)\,. 
		\]
		Since $\muhat_{a_{t+1}}(t) \!\leq\! U_{a_{t+1}}(t) $, the monotony of the $\KL$ implies 
		\[
		\muhat_a(t) < \muhat_{a_{t+1}}(t) \,.
		\]
		Then from equation~(\ref{ucb_unimodal indexes Uat+1 and Ua}) we deduce 
		\[
		N_{a_{t+1}}(t)\KLof{\muhat_{a_{t+1}}(t)}{U_{a_{t+1}}(t)} \leq N_a(t) \KLof{\muhat_a(t)}{U_{a_{t+1}}(t)} + \log\!\left(N_a(t)\right)
		\]
		and 
		\[
		N_{a_{t+1}}(t) \leq \dfrac{\KLof{\muhat_a(t)}{U_{a_{t+1}}(t)}}{\KLof{\muhat_{a_{t+1}}(t)}{U_{a_{t+1}}(t)}}  N_a(t)+ \dfrac{\log\!\left(N_a(t)\right)}{\KLof{\muhat_{a_{t+1}}(t)}{U_{a_{t+1}}(t)}}  \,.
		\]
		Similarly, since $\muhat_{a_{t+1}}(t)\!\leq\! \muhat^\star(t) \!=\! U_{\ahat^\star_t}(t)\!\leq\! U_{a_{t+1}}(t) $, the monotony of the $\KL$ implies 
		\[
		\KLof{\muhat_{a_{t+1}}(t)}{U_{a_{t+1}}(t)} \geq \KLof{\muhat_{a_{t+1}}(t)}{\muhat^\star(t)} \,.
		\]
		This implies 
		\[
		N_{a_{t+1}}(t) \leq \dfrac{\KLof{\muhat_a(t)}{U_{a_{t+1}}(t)}}{\KLof{\muhat_{a_{t+1}}(t)}{U_{a_{t+1}}(t)}}  N_a(t)+ \dfrac{\log\!\left(N_a(t)\right)}{\KLof{\muhat_{a_{t+1}}(t)}{\muhat^\star(t)}} \,.
		\]
		Since $\muhat_a(t) \!\leq\! \muhat_{a_{t+1}}(t) \!\leq\! \muhat^\star(t) \!\leq\! U_{a_{t+1}}(t) $, we have from Lemma~\ref{ucb_unimodal KL}:
		\[ \dfrac{\KLof{\muhat_a(t)}{U_{a_{t+1}}(t)}}{\KLof{\muhat_{a_{t+1}}(t)}{U_{a_{t+1}}(t)}} \leq  \dfrac{\KLof{\muhat_a(t)}{\muhat^\star(t)}}{\KLof{\muhat_{a_{t+1}}(t)}{\muhat^\star(t)}} \,.
		\]
		This implies
		\[  N_{a_{t+1}}(t) \KLof{\muhat_{a_{t+1}}(t)}{\muhat^\star(t)} \leq   N_a(t) \KLof{\muhat_a(t)}{\muhat^\star(t)}+ \log\!\left(N_a(t)\right)\,.
		\]

	\end{proof}

	\begin{lemma} [Empirical upper bounds]\label{ucb_unimodal empirical upper bounds}
		Under \KLUCBUB at each step time $t \!\geq \!1$,
		\[
		N_{a_{t+1}}(t) \KLof{\muhat_{a_{t+1}}(t)}{\muhat^\star(t)} \leq \log(t) \,.
		\]
		
	\end{lemma}
	
	\begin{proof} From equation~(\ref{ucb_unimodal index Uat+1}) we deduce 
		\beqan N_{a_{t+1}}(t)\KLof{\muhat_{a_{t+1}}(t)}{U_{a_{t+1}}(t)} &\leq&  N_{a_{t+1}}(t) \KLof{\muhat_{a_{t+1}}(t)}{U_{a_{t+1}}(t)} + \log\!\left(N_{a_{t+1}}(t)\right) \\
		&=& \log\!\left(N_{\ahat^\star_t}(t)\right)\\
		&\leq& \log(t) \,.
		\eeqan
		Furthermore, according to the followed strategy, we have
		\[
		\muhat_{a_{t+1}}(t) \leq \muhat^\star(t) =  U_{\ahat^\star_t}(t) \leq U_{a_{t+1}}(t) \,.
		\]
		Then the monotony of the $\KL$ implies 
		\[
		\KLof{\muhat_{a_{t+1}}(t)}{\muhat^\star(t)} \leq \KLof{\muhat_{a_{t+1}}(t)}{U_{a_{t+1}}(t)}
		\]
		and
		\[
		N_{a_{t+1}}(t) \KLof{\muhat_{a_{t+1}}(t)}{\muhat^\star(t)}\leq  N_{a_{t+1}}(t) \KLof{\muhat_{a_{t+1}}(t)}{U_{a_{t+1}}(t)} \leq  \log\!\left(N_{\ahat^\star_t}(t)\right)\leq \log(t) \,.
		\]
	\end{proof}
	\begin{lemma}\label{ucb_unimodal KL} Let $0\!\leq\!\mu \!\leq\! \mu' \!\leq\! \mu''\!\leq\!1 $. We have:
		\[
		\forall u \in [\mu'',1], \quad    \dfrac{\KLof{\mu}{u}}{\KLof{\mu'}{u}} \leq \dfrac{\KLof{\mu}{\mu''}}{\KLof{\mu'}{\mu''}} \,.
		\]
			
		\end{lemma}
		We prove this lemma only for Bernoulli distributions when $\KL(\cdot|\cdot) \!=\! \kl(\cdot|\cdot)$. The proof is simpler for Gaussian distributions.
		\begin{proof}[For Bernoulli distributions] We denote by $\frac{\partial\kl}{\partial p}(\cdot|\cdot)$ and $\frac{\partial\kl}{\partial q}(\cdot|\cdot)$ the derivatives of $\kl(\cdot|\cdot)$ respectively  according to the first and second variables. Let us consider $0\!\leq\!\mu \!\leq\! \mu' \!\leq\! \mu''\!<\!1 $ and $f\!:u \!\in\!(\mu'',1) \!\mapsto\! \klof{\mu'}{\mu''}  \klof{\mu}{u} \!-\! \klof{\mu}{\mu''} \klof{\mu'}{u}$. $f$ is a C-1 function and for $u \!\in\!(\mu'',1)$,
		\[
		f'(u) = \klof{\mu'}{\mu''}\frac{\partial\kl}{\partial q}(\mu|u) - \klof{\mu}{\mu''}\frac{\partial\kl}{\partial q}(\mu'|u) = \dfrac{\klof{\mu'}{\mu''}(u-\mu) - \klof{\mu}{\mu''}(u - \mu') }{u(1-u)}\,.
		\]
		Let us introduce $g\!:u \!\in\!(\mu'',1) \!\mapsto\! \klof{\mu'}{\mu''}(u\!-\!\mu) \!-\! \klof{\mu}{\mu''} (u \!-\! \mu')$. $g$ is a C-1 function and for $u \!\in\!(\mu'',1)$,
		\[
		g'(u) = \klof{\mu'}{\mu''}-\klof{\mu}{\mu''}\,.
		\]
		Since $\mu \!\leq\! \mu' \!\leq\! \mu''$, the monotony of the $\kl$ implies $g'(u) \!\leq\! 0$. Then $g$ is a non-increasing function. In addition 
		\[
		g(\mu'') = \klof{\mu'}{\mu''}(\mu'' - \mu) - \klof{\mu}{\mu''}(\mu'' - \mu') = (\mu''-\mu) \times \left(\klof{\mu'}{\mu''} - \dfrac{\klof{\mu}{\mu''}}{(\mu'' - \mu)} (\mu'' - \mu') \right) \,.
		\]
		Lastly, let us consider $h\!:p \!\in\![0,\mu'') \!\mapsto\! \frac{\kl(p|\mu'')}{\mu'' - p}$. $h$ is a C-1 function or $p \!\in\![0,\mu'')$,
		\[
		h'(p) =\dfrac{(\mu'' - p) \frac{\partial\kl}{\partial p}(p|\mu'') + \kl(p|\mu'')}{(\mu'' - p)^2} = \dfrac{ - \kl(\mu''|p)}{(\mu'' - p)^2} \leq 0 \,.
		\]
		Then $h$ is a non-increasing function. In particular, since $\mu\!\leq\!\mu'$,
		\[
		\dfrac{\klof{\mu}{\mu''}}{(\mu'' - \mu)} \geq \dfrac{\klof{\mu'}{\mu''}}{(\mu'' - \mu')} \,.
		\]
		This implies $g(\mu'')\!\leq\!0$ and $g\!\leq\!0$, since $g$ is a non-increasing function. Then $f' \!\leq\!0$ and $f$ is a non-increasing function. Since $f(\mu'')\!=\!0$, this implies $f \!\leq\!0$, which ends the proof.
		\end{proof}
	
	\subsection{Reliable current best arm and  means}
	
	%\oam{Can't we say that this section mostly follows classical proof strategies, controlling occurrence of bad events by resorting to concentration inequalities, etc ? Should we put the whole section in appendix and only keep here the statement of the technical Lemmas? If there is some novelty here, it is not well highlighted.}
	As in \IMEDUB analysis, we consider the subset $\cT_\epsilon$ of times where everything is well behaved, that is: the current best arm corresponds to the true one and the empirical means of the best arm and the current chosen arm are $\epsilon$-accurate for $0<\epsilon<\epsilon_\nu$, i.e.
	\[
		\cT_\epsilon \coloneqq \Set{t \geq 1:\ \Ahat^\star(t) = \Set{a^\star} \ad  \forall a \in \Set{ a^\star,a_{t+1} },\ \abs{\muhat_a(t) - \mu_a } < \epsilon }\,.
		\]
	 We will show that its complementary set is finite on average. In order to prove this we decompose the set $\cT_\epsilon$ in the following way. Let $\cE_\epsilon$ be the set of times where the means are well estimated,
	 \[
		\cE_\epsilon \coloneqq \Set{t \geq 1:\ \forall a \in \Ahat^\star(t)\cup\Set{a_{t+1} },\ \abs{\muhat_a(t) - \mu_a } < \epsilon }\,,
	  \]
	 and $\Lambda_\epsilon$ the set of times where an arm that is not the current optimal neither pulled is underestimated
	 {\small 
			\[
			\Lambda_\epsilon \!\coloneqq \!\Set{\!t \geq 1 :\, \exists a \in \cV_{\ahat^\star_t}\!\setminus\!\Set{a_{t+1},\ahat^\star_t} \textnormal{\,s.t.\,} \left\{\begin{array}{l}
				\!\muhat_a(t) < \mu_a - \epsilon\\
				\!\log\!\left(N_{a_{t+1}}(t)\right) \leq  N_a(t) \KLof{\muhat_a(t)}{ \mu_a - \epsilon}  + \log\!\left((N_{a}(t)\right)   
				\end{array} \right.  \!\!}\,.
			\]
		}
~\\Then, in the same way as for Lemma~\ref{lem:decomposition_T_epsilon}, we can prove the following inclusion. 		
\begin{lemma}[Relations between the subsets of times]\label{lem:ucb decomposition_T_epsilon}
For $ 0 < \epsilon < \epsilon_\nu$,
\begin{equation}
    \cT_\epsilon^c\cap\Set{t \geq 1:\ \gamma_t = 1} \cap\cE_\epsilon  \subset \Lambda_\epsilon   \,.
\label{eq:ucb decomp_T_epsilon}
\end{equation}
\end{lemma}
We can now resort to classical concentration arguments in order to control the size of these sets, which 
yields the following upper bounds. 
\begin{lemma}[Bounded subsets of times]For $ 0 < \epsilon < \epsilon_\nu$, 
 \[\Esp_\nu[\abs{\cE_\epsilon^c}] \leq \dfrac{10(d+1)\abs{\cA}}{\epsilon^4} \qquad		\Esp_\nu[\abs{\Lambda_\epsilon}] \leq 23d^2\abs{\cA}\dfrac{\log(1/\epsilon)}{\epsilon^6} \qquad \Esp_\nu[\abs{\Set{t \geq 1:\ \gamma_t = 0}\cap\cE_\epsilon}] \leq \dfrac{\abs{\cA}}{\epsilon^4} \,,
 \]
 where $d$ is the maximum degree of nodes in $G$.
 \label{lem: ucb ce_and_lambda_are_finite}
\end{lemma}
\begin{proof}
Refer to Lemma~\ref{lem:ce_and_lambda_are_finite} to prove $\Esp_\nu[\abs{\cE_\epsilon^c}] \!\leq\! \frac{10(d+1)\abs{\cA}}{\epsilon^4},\,	\Esp_\nu[\abs{\Lambda_\epsilon}] \!\leq\! 23d^2\abs{\cA}\frac{\log(1/\epsilon)}{\epsilon^6}$. It is exactly the same proof.

~\\Let $t \!\in\! \Set{t \geq 1:\ \gamma_t = 0}\cap\cE_\epsilon $. Then $ \gamma_t \!=\! 0 $. This implies
			\[
			a_{t+1} \neq \ahat^\star_t \ad \log\!\left(N_{a_{t+1}}(t)\right) > N_{a_{t+1}}(t) \KLof{\muhat_{a_{t+1}}(t)}{\muhat^\star(t)} \,.
			\]
			Since $ t \!\in\! \cE_\epsilon$ and $\epsilon \!<\! \epsilon_\nu$, we have 
			\[
			\abs{\muhat_{\ahat^\star_t}(t) - \mu_{\ahat^\star_t}} < \epsilon \ad \muhat_{a}(t)  \leq \muhat^\star(t) = \muhat_{\ahat^\star_t}(t) < \mu_{\ahat^\star_t} + \epsilon < \mu_{a} - \epsilon \,.
			\]
			This implies by Pinsker's inequality  
			\[
			\KLof{\muhat_{a_{t+1}}(t)}{\muhat^\star(t)} \geq \min\!\left(2(\muhat^\star(t) - \muhat_{a_{t+1}}(t))^2, \dfrac{(\muhat^\star(t) - \muhat_{a_{t+1}}(t))^2}{2} \right) > (2\epsilon)^2/2= 4 \epsilon^2 \,.
			\]
			In addition, for all $N \!\geq\! 1$, $\log(N) \!\leq\! 2\sqrt{N}$. Thus we have 
			\[
			2 \sqrt{N_{a_{t+1}}(t))} > 4 N_{a_{t+1}}(t) \epsilon^2 \quad \textnormal{i.e.}\quad N_{a_{t+1}}(t) < \dfrac{1}{4\epsilon^4} < \dfrac{1}{\epsilon^4}\,.
			\]
			This implies 
			\beqan
			\abs{\Set{t \geq 1:\ \gamma_t = 0}\cap\cE_\epsilon} 
			&\leq& \sum\limits_{t \geq 1}\ind_{\Set{ N_{a_{t+1}}(t) < 1/\epsilon^4}} \\
			&=& \suma\sum\limits_{t \geq 1}\ind_{\Set{ a_{t+1} = a \ad N_{a}(t) < 1/\epsilon^4}} \\
			&\leq& \suma\dfrac{1}{\epsilon^4} = \dfrac{\abs{\cA}}{\epsilon^4} \,.
			\eeqan
			
\end{proof}
Thus combining them with \eqref{eq:ucb decomp_T_epsilon} we obtain 
\beqan
\Esp_\nu[\abs{\cT_\epsilon^c}] &\leq& \Esp_\nu[\abs{\cE_\epsilon^c}] + \Esp_\nu[\abs{\Lambda_\epsilon}] + \Esp_\nu[\abs{\Set{t \geq 1:\ \gamma_t 
= 0}\cap\cE_\epsilon}] \\
&\leq& \dfrac{10(d+1)\abs{\cA}}{\epsilon^4} + 23d^2\abs{\cA}\dfrac{\log(1/\epsilon)}{\epsilon^6} +  \dfrac{\abs{\cA}}{\epsilon^4} \\
&\leq& 34d^2\abs{\cA} \dfrac{\log(1/\epsilon)}{\epsilon^6}   \,.
\eeqan
Indeed, we have
$$ \cT_\epsilon^c \subset \left(\cT_\epsilon^c\cap\Set{t \geq 1:\ \gamma_t = 1} \cap\cE_\epsilon\right)\cup \left(\Set{t \geq 1:\ \gamma_t  = 0} \cap\cE_\epsilon\right)\cup \cE_\epsilon^c \,.$$
Hence, we just proved the following lemma.
\begin{lemma}[Reliable estimators] \label{ucb unimodal reliability}For $ 0 < \epsilon < \epsilon_\nu$,
\[
\Esp_\nu[\abs{\cT_\epsilon^{c}}] \leq 34 d^2 \abs{\cA}\dfrac{\log(1/\epsilon)}{\epsilon^6} \,,
\]
 where $d$ is the maximum degree of nodes in $G$.
\end{lemma}

\subsection{Upper bounds on the numbers of pulls of sub-optimal arms}
	In this section, we now combine the different results of the previous sub-sections to prove Theorem~\ref{th:upper bounds}.
	
	\begin{proof}[Proof of Theorem~\ref{th:upper bounds}.] Please refer to Section~\ref{subsec: proof theorem}. It is exactly the same proof.
	\end{proof}

\section{\IMEDUB finite time analysis}
	\label{app : d-imed_analysis}
	In this section we assume that $G$ is a tree. \dIMEDUB behaves as \IMEDUB except during second order exploration phases. Thus, \dIMEDUB strategy  implies the same lower bounds and empirical upper bounds on the numbers of pulls as \IMEDUB strategy most of times. Then similar guaranties as those obtained under \IMEDUB can be established based on the same reasoning for \dIMEDUB. These guaranties involve the numbers of pulls of arms in $\cV_{a^\star}$ which are shown to be of order $ \cO\!\left(\log(T)\right)$, and the assumption that $G$ is a tree ensures the best arms $(\aul_t)_{t\geq1}$ of the sub-trees $\big(\hat G_{\aul_t}(t)\big)_{t\geq1}$
	belong to $\cV_{a^\star}$ most of times. Then, since $\cS_t$ is built as a sub-tree of $\hat G_{\aul_t}(t)$ that contains $\aul_t$ for all time step $t \!\geq\! 1$,  the \IMED type strategy followed during the second order exploration phases implies that exploration outside $\cV_{a^\star}$ is of order $\cO\!\left(\log\!\left(\cO\!\left(\log(T)\right)\right)\right)\!=\!\cO\!\left(\log\!\log(T)\right)$. 
	
	%\oam{Good, but perhaps highlight more what is novel here; we take a different path then ... etc Also explain that we only provide the main proof of IMED  here, and defer the one of UCB to the appendix, as it follows essentially the same ideas/steps ?}

	%\oam{We may be short on space to include the full proof in the paper. Would it be possible to 
	%	present all the lemmas but without their proof, and then explain how to combine them to obtain the final result? Then perhaps we can simply detail the most important lemmas and leave the more technical ones to the appendix?}
	
	\subsection{ Notations}
		Please, refer to Section~\ref{imed_unimodal notations}.
	\subsection{Strategy-based empirical bounds} 
		In this subsection, we provide empirical bounds very similar to the ones induced by \IMEDUB strategy.
	\begin{lemma}[Empirical lower bounds]\label{d-imedub unimodal empirical lower bounds}Under \dIMEDUB, at each step time $t \!\geq\! 1$, 
	\[1.\quad\forall a \in \cV_{\ahat_t^\star},\ \log\!\left(N_{a_{t+1}}(t)\right)  \leq N_{a}(t)\, \KLof{\muhat_{a}(t)}{\muhat^\star(t)} + \log\!\left(N_{a}(t)\right)  
	\textnormal{ and }  
	N_{a_{t+1}}(t) \leq N_{\ahat^\star_t}(t)\,. 
	\]
	Furthermore, if $a_{t+1} \!\notin\! \Set{\ahat^\star_t}\!\cup\!\cV_{\ahat^\star_t}$, we have 
	\[2.\quad\forall a \in \cS_t,\ \log\!\left(N_{a_{t+1}}(t)\right)  \leq N_{a}(t)\, \KLpof{\muhat_{a}(t)}{\muhat_{\aul_t}(t)} + \log\!\left(N_{a}(t)\right)  
	\textnormal{ and }  
	N_{a_{t+1}}(t) \leq N_{\aul_t}(t) \leq N_{\ahat^\star_t}(t)\,. 
	\]
	\end{lemma}
	
	\begin{proof} 
	~\\ \underline{\textbf{Case 1}} : $a_{t+1} \!\in\! \Set{\ahat^\star_t}\!\cup\!\cV_{\ahat^\star_t}$.
	~\\ This means there is no second order exploration at time $t$ and \dIMEDUB behaves as \IMEDUB. Then point 1. is satisfied according to Lemma~\ref{unimodal empirical lower bounds}.
	
	~\\ \underline{\textbf{Case 2}} : $a_{t+1} \!\notin\! \Set{\ahat^\star_t}\!\cup\!\cV_{\ahat^\star_t}$. 
	~\\ This means $\aul_t \neq \ahat^\star_t$ and according to \dIMEDUB strategy 
	\[
	\forall a \in \cS_t,\quad \log\!\left(N_{a_{t+1}}(t)\right) \leq I_{a_{t+1}}^{(\aul_t)}(t) \leq I_{a}^{(\aul_t)}(t) =  N_{a}(t)\, \KLpof{\muhat_{a}(t)}{\muhat_{\aul_t}(t)} + \log\!\left(N_{a}(t)\right)   \,. 
	\]
	Since $\aul_t \!\in\! \cS_t$ and $I_{\aul_t}^{(\aul_t)}(t) \!=\! \log\!\left(N_{\aul_t}(t)\right) $, by taking the $\exp(\cdot)$ we get $N_{a_{t+1}}(t) \!\leq\!N_{\aul_t}(t)$ and prove point 2.\,.
	Furthermore, still according to \dIMEDUB strategy, we have
	\[
	\forall a \in \Set{\ahat^\star_t}\cup\cV_{\ahat^\star_t},\quad \log\!\left(N_{a_{t+1}}(t)\right) \leq \log\!\left(N_{\aul_t}(t)\right) \leq I_{\aul_t}(t) \leq I_{a}(t) =  N_{a}(t)\, \KLof{\muhat_{a}(t)}{\muhat^\star(t)} + \log\!\left(N_{a}(t)\right)  \,.
	\]
	Since $I_{\ahat^\star_t}(t) \!=\! \log\!\left(N_{\ahat_t^star}(t)\right)$, by taking the $\exp(\cdot)$ we get in particular $N_{a_{t+1}}(t)\!\leq\! N_{\ahat^\star_t}(t) $ and prove point 1.\,.

	\end{proof}
	\begin{lemma}[Empirical upper bounds]\label{d-imed unimodal empirical upper bounds}
		Under \dIMEDUB at each step time $t \!\geq\! 1$,
		\[
	1. \quad	N_{\aul_t}(t) \,\KLof{\muhat_{\aul_t}(t)}{\muhat^\star(t)} \leq \log(t) \,. 
		\]
	Furthermore, if $a_{t+1} \!\notin\! \Set{\ahat^\star_t}\!\cup\!\cV_{\ahat^\star_t}$, we have 
	\[
	2. \quad	N_{a_{t+1}}(t) \,\KLpof{\muhat_{a_{t+1}}(t)}{\muhat_{\aul_t}(t)} \leq \log\!\left(\dfrac{\log(t)}{\KLof{\muhat_{\aul_t}(t)}{\muhat^\star(t)}}\right) \,.
		\]
	\end{lemma}
	
	\begin{proof} 1. According to the followed strategy, we have
		\[
		I_{\aul_t}(t) \leq I_{\ahat^\star_t}(t) \,.
		\]
		It remains, to conclude, to note that
		\[
		N_{\aul_t}(t) \KLof{\muhat_{\aul_t}(t)}{\muhat^\star(t)}  
		\leq I_{\aul_t}(t)\,, 
		\]
		and \[
		I_{\ahat^\star_t}(t) =  \log(N_{\ahat^\star_t}(t)) \leq \log(t) \,.
		\]
	2. We assume that $a_{t+1} \!\notin\! \Set{\ahat^\star_t}\!\cup\!\cV_{\ahat^\star_t}$. According to the followed strategy, we have
		\[
		I_{a_{t+1}}^{(\aul_t)}(t) \leq I_{\aul_t}^{(\aul_t)}(t) \,.
		\]
	Furthermore, by definition of the second order \IMED indexes we have
		\[
		N_{a_{t+1}}(t) \KLpof{\muhat_{a_{t+1}}(t)}{\muhat_{\aul_t}(t)}  
		\leq I_{a_{t+1}}^{(\aul_t)}(t)\,, \]
		and 
		\[
		I_{\aul_t}^{(\aul_t)}(t) =  \log(N_{\aul_t}(t)) \,.
		\]
	We conclude the proof using point 1. we  just proved.
	\end{proof}

\subsection{Reliable current best arm and  means}
	
	%\oam{Can't we say that this section mostly follows classical proof strategies, controlling occurrence of bad events by resorting to concentration inequalities, etc ? Should we put the whole section in appendix and only keep here the statement of the technical Lemmas? If there is some novelty here, it is not well highlighted.}
	As in \IMEDUB analysis, we consider the subset $\cT_\epsilon$ of times where everything is well behaved, that is: the current best arm corresponds to the true one and the empirical means of the best arm, the arm with minimal current index and the current chosen arm are $\epsilon$-accurate for $0<\epsilon<\epsilon_\nu$, i.e.
	\[
		\cT_\epsilon \coloneqq \Set{t \geq 1:\ \Ahat^\star(t) = \Set{a^\star} \ad  \forall a \in \Set{ a^\star, \aul_t, a_{t+1} },\ \abs{\muhat_a(t) - \mu_a } < \epsilon }\,.
		\]
	 We will show that its complementary set is finite on average. In order to prove this we decompose the set $\cT_\epsilon$ in the following way. Let $\cE_\epsilon$ be the set of times where the means are well estimated,
	 \[
		\cE_\epsilon \coloneqq \Set{t \geq 1:\ \forall a \in \Ahat^\star(t)\cup\Set{\aul_t, a_{t+1} },\ \abs{\muhat_a(t) - \mu_a } < \epsilon }\,,
	  \]
	 and $\Lambda_\epsilon$ the set of times where an arm that is not the current optimal neither pulled is underestimated
	 {\small 
			\[
			\Lambda_\epsilon \!\coloneqq \!\Set{\!t \geq 1 :\, \exists a \in \cV_{\ahat^\star_t}\!\setminus\!\Set{a_{t+1},\ahat^\star_t} \textnormal{\,s.t.\,} \left\{\begin{array}{l}
				\!\muhat_a(t) < \mu_a - \epsilon\\
				\!\log\!\left(N_{a_{t+1}}(t)\right) \leq  N_a(t) \KLof{\muhat_a(t)}{ \mu_a - \epsilon}  + \log\!\left((N_{a}(t)\right)   
				\end{array} \right.  \!\!}\,.
			\]
		}
Then we get the same relation between these sets as for \IMEDUB strategy.	
\begin{lemma}[Relations between the subsets of times]\label{lem: d-imed decomposition_T_epsilon}
For $ 0 < \epsilon < \epsilon_\nu$,
\begin{equation}
    \cT_\epsilon^c\setminus\cE_\epsilon^c  \subset \Lambda_\epsilon   \,.
\label{eq: d-imed decomp_T_epsilon}
\end{equation}
\end{lemma}
\begin{proof} The proof is exactly the same as for Lemma~\ref{lem:decomposition_T_epsilon}.
\end{proof}
We can now resort to classical concentration arguments in order to control the size of these sets, which 
yields the following upper bounds. 
\begin{lemma}[Bounded subsets of times]For $ 0 < \epsilon < \epsilon_\nu$, 
 \[\Esp_\nu[\abs{\cE_\epsilon^c}] \leq \dfrac{10\abs{\cA}^2}{\epsilon^4} \qquad		\Esp_\nu[\abs{\Lambda_\epsilon}] \leq 23d^2\abs{\cA}\dfrac{\log(1/\epsilon)}{\epsilon^6} \,,
 \]
 where $d$ is the maximum degree of nodes in $G$.
 \label{lem: d-imed ce_and_lambda_are_finite}
\end{lemma}	
\begin{proof}
Refer to Lemma~\ref{lem:ce_and_lambda_are_finite} to prove $	\Esp_\nu[\abs{\Lambda_\epsilon}] \!\leq\! 23d^2\abs{\cA}\frac{\log(1/\epsilon)}{\epsilon^6}$. It is exactly the same proof.

Using Lemma~\ref{d-imedub unimodal empirical lower bounds} we have
		\[
		\forall t \geq 1, \quad N_{a_{t+1}}(t) \leq N_{\aul_t} \leq N_{\ahat^\star_t}(t) \,.
		\]
		Since $ \ahat^\star_t \!\in\!\argmin\limits_{\ahat^\star \in \Ahat^\star(t)}N_{\ahat^\star}(t)$, this implies
		\[
			\forall t \geq 1,\forall \ahat^\star \in \Ahat^\star(t),  \quad N_{a_{t+1}}(t) \leq N_{\ahat^\star_t}(t) \leq N_{\ahat^\star}(t) \,.
		\]
		Then, based on the concentration inequalities from Lemma~\ref{unimodal concentration},  we obtain
		\beqan 
		\Esp_\nu[\abs{\cE_\epsilon^c}] 
		&\leq& \sum\limits_{a, a' \in \cA} \Esp_\nu\!\left[\sum\limits_{t\geq1}{\ind_{\Set{a_{t+1}=a,\ N_{a'}(t) \geq  N_{a}(t),\  \abs{\muhat_{a'}(t) - \mu_{a'}} \geq \epsilon }}}\right] \\
		&\leq & \sum\limits_{a, a' \in \cA} \dfrac{10}{\epsilon^4} \\
		&\leq&  \dfrac{10\abs{\cA}^2}{\epsilon^4} \,.
		\eeqan

\end{proof}
Thus combining them with \eqref{eq: d-imed decomp_T_epsilon} we obtain 
\beqan
\Esp_\nu[\abs{\cT_\epsilon^c}] &\leq& \Esp_\nu[\abs{\cE_\epsilon^c}] + \Esp_\nu[\abs{\Lambda_\epsilon}]\\
&\leq& \dfrac{10\abs{\cA}^2}{\epsilon^4} + 23d^2\abs{\cA}\dfrac{\log(1/\epsilon)}{\epsilon^6}  \\
&\leq& 33 d\abs{\cA}^2 \dfrac{\log(1/\epsilon)}{\epsilon^6}   \,.
\eeqan
Indeed, we have
$$ \cT_\epsilon^c \subset \left(\cT_\epsilon^c\setminus\cE_\epsilon^c\right)\cup \cE_\epsilon^c \,.$$
Hence, we just proved the following lemma.
\begin{lemma}[Reliable estimators] \label{d-imedub unimodal reliability}For $ 0 < \epsilon < \epsilon_\nu$,
\[
\Esp_\nu[\abs{\cT_\epsilon^{c}}] \leq 33 d\abs{\cA}^2 \dfrac{\log(1/\epsilon)}{\epsilon^6} \,,
\]
 where $d$ is the maximum degree of nodes in $G$.
\end{lemma}
\subsection{\label{subsec: d-imedub proof theorem}Upper bounds on the numbers of pulls of sub-optimal arms}
	In this section, we now combine the different results of the previous sub-sections to prove Theorem~\ref{th:upper bounds}.

	\begin{proof}[Proof of Theorem~\ref{th:upper bounds}.] From Lemma~\ref{d-imedub unimodal reliability}, considering the following subset of times
	\[
		\cT_\epsilon \coloneqq \left\{\begin{array}{l}
		    \hspace{-2mm} t \geq 1:\  \Ahat^\star(t) = \Set{a^\star}  \\
		     \hspace{10mm}\forall a \in \Set{ a^\star, \aul_t, a_{t+1} },\ \abs{\muhat_a(t) - \mu_a } < \epsilon 
		\end{array} \! \right\}\,. 
		\]
		we have 
		\[
		\Esp_\nu[\abs{\cT_\epsilon^{ c}}] \leq 33d\abs{\cA}^2\dfrac{\log(1/\epsilon)}{\epsilon^6} \,,
		\]
		where $d$ is the maximum degree of nodes in $G$.
		Then, let us consider $ a \!\neq\! a^\star$ and a time step $t \!\in\! \cT_\epsilon$ such that $a_{t+1} \!=\! a$. Since $ t \!\in\! \cT_\epsilon $, we have
		\[
		\ahat^\star_t = a^\star \ad \abs{\muhat_a(t) - \mu_a}, \abs{\muhat_{\aul_t}(t) -\mu_{\aul_t}},  \abs{\muhat_{a^\star}(t) - \mu_{a^\star}} < \epsilon \,.
		\]
		Then $\aul_t \neq a^\star$ and, by construction of $\alpha_\nu(\cdot)$ (see Section~\ref{imed_unimodal notations} \!Notations),
		\[
		\KLof{\muhat_{\aul_t}(t)}{\muhat^\star(t)} = \KLof{\muhat_{\aul_t}(t)}{\muhat_{a^\star}(t)} \geq \dfrac{\KLof{\mu_{\aul_t}}{\mu_{a^\star}}}{1 + \alpha_\nu(\epsilon)} \,. \] 
		
		~\\ \underline{\textbf{Case 1}} : $a_{t+1} \!\in\! \Set{\ahat^\star_t}\!\cup\!\cV_{\ahat^\star_t}$, that is $a = \aul_t \in \cV_{a^\star}$
		~\\ Then from Lemma~\ref{d-imed unimodal empirical upper bounds} we get
		\[
		N_a(t) \KLof{\muhat_a(t)}{\muhat^\star(t)} \leq \log(t) \leq \log(T) \,,
		\]
		This implies 
		\[N_{a}(t) \leq \dfrac{1 + \alpha_\nu(\epsilon)}{\KLof{\mu_a}{\mu_{a^\star}}} \log(T)  \,.
		\]
		
		~\\ \underline{\textbf{Case 2}} : $a_{t+1} \!\notin\! \Set{\ahat^\star_t}\!\cup\!\cV_{\ahat^\star_t}$, that is $a \in G_{\aul_t}\!\!\setminus\!\Set{\aul_t}$ and $\aul_t \in \cV_{a^\star}$
		~\\ Then from Lemma~\ref{d-imed unimodal empirical upper bounds} we get
		\[
		N_a(t) \KLpof{\muhat_a(t)}{\muhat_{\aul_t}(t)} \!\leq\! \log\!\left(\dfrac{\log(t)}{\KLof{\muhat_{\aul_t}(t)}{\muhat^\star(t)}}\right)\!\leq\!\log\!\left(\dfrac{1 + \alpha_\nu(\epsilon)}{\KLof{\mu_{\aul_t}}{\mu_{a^\star}}} \log(T)\!\!\right) \!\leq\! \log\!\left(\dfrac{1 + \alpha_\nu(\epsilon)}{\min\limits_{\aul \in \cV_{a^\star}}\KLof{\mu_{\aul}}{\mu_{a^\star}}} \log(T)\!\!\right).
		\]
		Since $G$ is a tree and $a \in G_{\aul_t}\!\!\setminus\!\Set{\aul_t}$, we have $\mu_a < \mu_{\aul_t}$. Since $ \epsilon < \epsilon_\nu$, we have $ \muhat_a(t) < \mu_a + \epsilon < \mu_{\aul_t} -\epsilon < \muhat_{\aul_t}(t)  $ and $\KLpof{\muhat_a(t)}{\muhat_{\aul_t}(t)} = \KLof{\muhat_a(t)}{\muhat_{\aul_t}(t)}$.
		By construction of $\alpha_\nu(\cdot)$, it comes
		\[
		\KLof{\muhat_a(t)}{\muhat_{\aul_t}(t)} = \KLof{\muhat_a(t)}{\muhat_{\aul_t}(t)} \geq \dfrac{\KLof{\mu_a}{\mu_{\aul_t}}}{1 + \alpha_\nu(\epsilon)} \geq \dfrac{1}{1 + \alpha_\nu(\epsilon)} \min\limits_{\aul \in \cV_{a^\star}}\KLof{\mu_a}{\mu_{\aul}}\,.\] 
		Then we have  
		\[N_{a}(t) \leq \dfrac{1 + \alpha_\nu(\epsilon)}{\min\limits_{\aul \in \cV_{a^\star}}\KLof{\mu_a}{\mu_{\aul}}} \log\!\left(\dfrac{1 + \alpha_\nu(\epsilon)}{\min\limits_{\aul \in \cV_{a^\star}}\KLof{\mu_{\aul}}{\mu_{a^\star}}} \log(T)\!\!\right)  \,.
		\]

		~\\ Thus, we have shown that for $a \!\neq\! a^\star$, for all $t \!\in\! \cT_\epsilon $ such that $ a_{t+1} = a $, 
		\[
		N_{a}(T) \!\leq\! \left\{\hspace{-2mm}\begin{array}{ll}
		\dfrac{1 \!+\! \alpha_\nu(\epsilon)}{\KLof{\mu_a}{\mu_{a^\star}}} \log(T)  &\hspace{-3mm}, \textnormal{ if } a \in \cV_{a^\star} \\
	\dfrac{1 + \alpha_\nu(\epsilon)}{\min\limits_{\aul \in \cV_{a^\star}}\KLof{\mu_a}{\mu_{\aul}}} \log\!\left(\dfrac{1 + \alpha_\nu(\epsilon)}{\min\limits_{\aul \in \cV_{a^\star}}\KLof{\mu_{\aul}}{\mu_{a^\star}}} \log(T)\!\!\right) &\hspace{-3mm}, \textnormal{ otherwise}.
		\end{array}\right. 
		\]
		This implies:
	\[
		N_{a}(T) \!\leq\! \left\{\hspace{-2mm}\begin{array}{ll}
		\dfrac{1 \!+\! \alpha_\nu(\epsilon)}{\KLof{\mu_a}{\mu_{a^\star}}} \log(T)   + \abs{\cT_\epsilon^c} + 1 &\hspace{-3mm}, \textnormal{ if } a \in \cV_{a^\star} \\
	\dfrac{1 + \alpha_\nu(\epsilon)}{\min\limits_{\aul \in \cV_{a^\star}}\KLof{\mu_a}{\mu_{\aul}}} \log\!\left(\dfrac{1 + \alpha_\nu(\epsilon)}{\min\limits_{\aul \in \cV_{a^\star}}\KLof{\mu_{\aul}}{\mu_{a^\star}}} \log(T)\!\!\right)  + \abs{\cT_\epsilon^c} + 1  &\hspace{-3mm}, \textnormal{ otherwise}.
		\end{array}\right. 
		\]
		Averaging these inequalities allows us to conclude.
	\end{proof}

\section{Details on numerical experiments}
\label{app: exp}
In this section we briefly describe how the subsets $(\cS_t)_{t\geq1}$ used in \dIMEDUB are dynamically chosen for the experiments.
We assume in this section that $\cA \!=\! \llbracket 1,A\rrbracket$ with $A \!\geq\! 2$.  

Let us introduce  the function $d(\cdot)$ that extracts  dichotomously from an interval $\llbracket a, a'\rrbracket$ a subset of arms from their extreme values to its median.
 \begin{algorithm}[H]
		\caption{Dichotomous function $d(\cdot)$}
		\begin{algorithmic}
		\INPUT $\llbracket a, a' \rrbracket \subset \cA $, where $a < a'$
		\IF{$a' - a $ < 4}
		\STATE \textbf{return} $ \llbracket a, a' \rrbracket $ 
		\ELSE
		\STATE \textbf{return} $\Set{a, a'} \cup d\!\left(\llbracket a +  \lfloor (a' -a)/4\rfloor , a' - \lfloor (a' - a)/4\rfloor   \rrbracket\right) $ 
		\ENDIF
	\end{algorithmic}
\end{algorithm}
Using function $d(\cdot)$ we dynamically build a sequence of subsets $(\tilde \cS_t)_{t \geq1}$ as described in Algorithm~\ref{alg: s}, where:
\begin{itemize}[label=-, leftmargin = *, topsep = 0cm, parsep = 0cm , itemsep = 0cm]
    \item $median(\cdot)$ returns the median of input subset,
    \item $list(\cdot)$ creates the list (indexed from $1$) of input elements,
    \item $index(e,L)$ returns the index of element $e$ in list $L$,
    \item $element(I,L)$ returns the elements of list $L$ with indexes in $I$,
    \item $distance(a, \cS) = \min_{a' \in \cS} \abs{a' - a}$, for all arm $a \in \cA$ and all subset of arms $\cS \subset \cA$,
    \item $ append(e, L)$ returns list $L$ to which is added element $e$.
\end{itemize}
 \begin{algorithm}[H]
 
		\caption{Dynamic sequence of subsets $(\tilde \cS_t)_{t\geq1}$}
		\begin{algorithmic}
		\label{alg: s}
		\STATE $ \tilde \cS_1 \gets d(\cA) $
		\STATE $ \tilde a^\star_1 \gets median(\cS_1) $
		\STATE $ \text{List}_{\tilde\cS} \gets list(\tilde\cS_1) $
		\STATE $ \text{List}_{\tilde a^\star} \gets list(\tilde a^\star_1) $
		\FOR{$t = 2 \dots T$}
		\IF{$\ahat^\star_t \in \tilde\cS_{t-1}$}
		\STATE $\tilde a^\star_t \gets \ahat^\star_t$
		\IF{$\tilde a^\star_t \in \text{List}_{\tilde  a^\star} $}
		\STATE $i \gets index(\ahat^\star_t, \text{List}_{\tilde a^\star})$
		\STATE $\tilde\cS_t \gets element(i, \text{List}_{\tilde \cS})$
		\vspace{2mm}\STATE $\text{List}_{\tilde\cS} \gets element(\llbracket 1, i \rrbracket, \text{List}_{\tilde \cS})   $
		\STATE $\text{List}_{\tilde a^\star} \gets element(\llbracket 1, i \rrbracket, \text{List}_{\tilde a^\star})   $
		\ELSE
		\STATE $ \Delta \gets distance(\tilde a^\star_t, \tilde\cS_{t-1}\!\setminus\!\Set{\tilde a^\star_t}) $
		\STATE $\tilde \cS_t \gets \tilde \cS_{t-1}\cup d(\llbracket \tilde a^\star_t - \Delta, \tilde a^\star_t + \Delta \rrbracket) $
		\vspace{2mm}\STATE $\text{List}_{\tilde\cS} \gets append(\tilde \cS_t, \text{List}_{\tilde \cS})   $
		\STATE $\text{List}_{\tilde a^\star} \gets append(\tilde a^\star_t, \text{List}_{\tilde a^\star})   $
		
		\ENDIF
		\ENDIF
		\ENDFOR
	\end{algorithmic}
\end{algorithm}

Then we build the sequence of subsets $(\cS_t)_{t\geq 1}$ as follows:
$$ \forall t \geq 1,\quad \cS_t = \left\{\begin{array}{ll}
\Set{\aul_t}\cup\Set{a \in \tilde \cS_t:\ a < \aul_t }     & \text{if } \aul_t < \ahat^\star_t \,, \\
 \Set{\aul_t}\cup\Set{a \in \tilde \cS_t:\ a > \aul_t }     & \text{if } \aul_t > \ahat^\star_t \,.
\end{array} \right. $$

%{\noindent \em Remainder omitted in this sample. See http://www.jmlr.org/papers/ for full paper.}

\vskip 0.2in
\bibliography{biblio}

\end{document}